%% file: main.tex
\documentclass{article} 
\usepackage{iclr2027_conference,times}
\usepackage[utf8]{inputenc}

\input{math_commands.tex}

\usepackage{hyperref}
\usepackage{url}
\usepackage{float}

\usepackage[short]{optidef}
\usepackage{bbold}

\usepackage{amsmath,amssymb,amsthm,mathtools}
\usepackage{booktabs}
\usepackage{array}
\usepackage{multirow}
\usepackage[shortlabels]{enumitem}
\usepackage{xcolor}
\usepackage{microtype}
\usepackage{subcaption}
\usepackage{nicefrac}
\usepackage{algorithm}
\usepackage{algpseudocode}
\usepackage{tikz}
\usetikzlibrary{arrows.meta,calc,patterns}
\usepackage{thmtools}
\usepackage{thm-restate}
\usepackage[capitalise,noabbrev,nameinlink]{cleveref}

\hypersetup{colorlinks=true,linkcolor=blue!60!black,citecolor=blue!60!black,
  urlcolor=blue!60!black}

\theoremstyle{plain}
\newtheorem{theorem}{Theorem}[section]

\newtheorem{lemma}[theorem]{Lemma}

\theoremstyle{definition}
\newtheorem{definition}[theorem]{Definition}

\theoremstyle{remark}

\crefname{claim}{Claim}{Claims}
\crefname{step}{Step}{Steps}

\newcommand{\relmiddle}[1]{\mathrel{}\middle#1\mathrel{}}

\newcommand{\GL}{\mathrm{GL}}
\newcommand{\Inst}{\gM}                        
\newcommand{\Sub}{\mathsf{S}}                  
\newcommand{\Red}{\mathsf{R}}
\newcommand{\sRed}{\mathsf{R}}                  
\newcommand{\Lmap}{\mL}                        
\newcommand{\hmap}{\vh}

\newcommand{\mPi}{\bm{\Pi}}
\newcommand{\mGam}{\bm{\Gamma}}                
\newcommand{\vgamma}{\bm{\gamma}}
\newcommand{\vlambda}{\bm{\lambda}}
\newcommand{\cblk}{\mathrm{C}}
\newcommand{\iblk}{\mathrm{I}}
\newcommand{\cp}[1]{#1_{\cblk}}                
\newcommand{\ip}[1]{#1_{\iblk}}                
\newcommand{\conv}{\operatorname{conv}}
\newcommand{\sg}{\operatorname{sg}}

\newcommand{\dist}{\operatorname{dist}}
\newcommand{\OPT}{\operatorname{OPT}}
\newcommand{\T}{\top}
\newcommand{\bmat}[4]{\begin{psmallmatrix}#1&#2\\#3&#4\end{psmallmatrix}}

\crefformat{equation}{(#2#1#3)}
\Crefformat{equation}{Equation~(#2#1#3)}

\title{Reformulation-contrastive learning for Mixed Integer Programs}

\author{
Ousema Bouaneni\textsuperscript{1},
Mathis Le Bail\textsuperscript{1},
Clément Elliker\textsuperscript{1},
Maël Jenny\textsuperscript{1,2},
Sonia Vanier\textsuperscript{1}
\\[3pt]
\textsuperscript{1} LIX (École Polytechnique, IP Paris, CNRS), France \\
\textsuperscript{2} AMIAD (Agence Ministérielle pour l’IA de Défense), France
\\[2pt]
\multicolumn{1}{c}{\texttt{ousema.bouaneni@polytechnique.edu}}
}

\iclrfinalcopy 
\begin{document}

\maketitle
\fancyhead{}
\renewcommand{\headrulewidth}{0pt}

\begin{abstract}
Mixed-integer linear programs (MILP) model many real-world decision problems, motivating machine-learning methods that exploit recurring structure to accelerate MILP solving. MILPs can admit many equivalent formulations: integrality-preserving changes of variables and the addition of redundant constraints can alter their formulations while preserving the optimization problem. We leverage these reformulations as a source of self-supervision for learning general-purpose representations of MILP variables and constraints. We characterize the affine reformulations that are valid for every input instance, and distinguish re-descriptions, which leave variables unchanged, from substitutions, which transform them predictably. Building on equivariant self-supervised learning, we introduce ReMILP (reformulation-contrastive MILP representation learning), which jointly trains a graph neural network and a hypernetwork to predict how variable embeddings transform under changes of variables. Without solver-derived labels, ReMILP learns representations that exhibit the intended invariance and equivariance on unseen problem classes. Across binary solution, constraint activity and integrality gap prediction, these representations carry task-relevant information when frozen and provide a useful initialization for fine-tuning.
\end{abstract}

\section{Introduction}
Mixed-integer linear programs (MILPs) are a flexible modeling framework for challenging decision problems arising in domains such as energy, transportation and logistics. Solving large MILPs remains computationally demanding, motivating a growing body of work that integrates machine learning into the solving pipeline, for instance to guide branching or to predict high-quality partial solutions \citep{Gasse2019-nips,Paulus2022-icml,Huang2024-icml}. These approaches motivate the development of general-purpose MILP representations that capture reusable problem structure beyond a single solver task.

A mixed-integer linear program has no canonical formulation. The same problem can be expressed using different units, binary conventions or variable parametrizations. The resulting formulations have different coefficients and bounds, but their solution sets are related by an affine map. This type of equivalence has been leveraged for contrastive representation learning on other structured modalities, including source code \citep{jain2021-emnlp}, boolean satisfiability \citep{duan2022-icml}, and logic circuits \citep{wang2025-tacd}. We extend this idea to MILPs, using equivalent reformulations as a source of self-supervision to pre-train an encoder that produces variable and constraint representations reusable across downstream MILP tasks. The closest self-supervised pre-training approach, FORGE \citep{Shafi2026-tmlr}, instead learns MILP representations through vector-quantized graph reconstruction. Reformulations have also been used for contrastive learning in MILP branching \citep{Lin2024-iclr,Lu2026-iclr}, but only as an auxiliary to supervised imitation learning and with transformations chosen to preserve the branching target. In contrast, we use reformulation equivalence itself as the supervision signal and consider a broader family of transformations. 

This broader family introduces an important subtlety. \emph{Re-descriptions} alter the constraint system while leaving the variables unchanged, so their embeddings should remain invariant. \emph{Substitutions}, i.e. type-preserving affine changes of variables, instead modify the variables themselves, so their embeddings should transform predictably. A standard contrastive objective would enforce invariance to both. We therefore adopt equivariant self-supervised learning \citep{Garrido2023-icml}, requiring variable embeddings to be \emph{invariant} to re-descriptions and \emph{equivariant} to substitutions. We characterize all valid affine reformulations and introduce ReMILP, which jointly learns a MILP encoder and the action of local substitutions on its variable embeddings, without solver-derived labels.

Our contributions are as follows:
\begin{itemize}[nosep]
\item We give necessary and sufficient conditions for a substitution or a re-description to be valid on every instance (\cref{thm:sub,thm:redesc}), and show that every valid affine augmentation rule is a composition of both (\cref{thm:complete}). In particular, integer variables admit arbitrary unimodular changes of variable, and
redundant constraints can be constructed from any nonnegative combination of existing constraints with loosened right-hand sides.

\item We introduce ReMILP (reformulation-contrastive MILP representation
learning), which jointly trains a graph neural network to produce variable embeddings that are invariant to re-descriptions and equivariant to substitutions, and a hypernetwork to predict how these embeddings transform under the corresponding changes of variables.

\item We validate the learned invariance and equivariance and evaluate downstream transfer on binary solution, constraint activity, and integrality gap prediction across Distributional MIPLIB \citep{Huang2024-arxiv} and MILP-Evolve \citep{Li2025-iclr}. In both frozen and fine-tuned regimes, we compare ReMILP with FORGE \citep{Shafi2026-tmlr} and random or training-from-scratch baselines.

\end{itemize}
\section{Related work}
\paragraph{Learning for MILP solving.}
Graph neural networks operating on the MILP representation of \citet{Gasse2019-nips} have been used to learn heuristic components of branch-and-bound solvers, such as branching policies \citep{Gasse2019-nips,Gupta2020-nips,Parsonson2023-aaai} or cut selection rules \citep{Paulus2022-icml,Wang2023-iclr}. Another line of work predicts properties of optimal solutions to restrict the search space before or during solving. Most approaches predict binary variable values at an optimum and use confident predictions to delimit the solver's search region \citep{Ding2020-aaai,Nair2020-arxiv,Han2023-iclr,Huang2024-icml,Liu2025-iclr}. Similarly, on the dual side, inequality constraints predicted to be tight at an optimum are replaced by equalities \citep{Li2026-iclr}. These approaches are supervised and task specific. We use binary solution, constraint activity and integrality gap prediction as downstream tasks to assess both the information captured by pre-trained representations and their value as an initialisation for fine-tuning. 
\paragraph{MILP representation learning.}
Few works aim to learn MILP representations that are both general-purpose and independent of solver labels. Multi-task and cross-distribution encoders share representations across tasks or domains, but are still trained on solver-derived labels \citep{Cai2025-cpaior,Pu2025-nips}. FORGE \citep{Shafi2026-tmlr} is, to our knowledge, the only self-supervised method for general-purpose MILP representations to date, and relies on a vector-quantized graph reconstruction objective. We pursue the same goal, but derive the self-supervision signal from equivalent reformulations.
\paragraph{Contrastive learning with augmented MILPs.}
Prior work uses reformulations to augment strong branching imitation learning, adding a contrastive term to the supervised loss. \citet{Lin2024-iclr} shift variables to generate instances that preserve strong-branching labels and contrast the graph-level embeddings. \citet{Lu2026-iclr} use sign flips, integer shifts, and redundant constraints obtained by summing pairs of constraints. Their contrastive loss does not directly pair reformulations. Instead, it clusters Branch-and-bound nodes by features and pulls together nodes within each cluster. Beyond these practical augmentations, they prove that any block-triangular change of variables with a signed permutation integer block and an invertible continuous block provides an equivalent instance.
\paragraph{Equivariant self-supervised learning.}
The tension between invariance and equivariance has been studied in visual self-supervised learning, where invariant objectives may discard information about the transformations relating two views. \citet{Dangovski2022-iclr} retain such information by predicting the applied transformation from the transformed embedding, making the representation sensitive to it without prescribing how it should change. \citet{Devillers2023-iclr} and \citet{Garrido2023-icml} instead learn the action itself: the former conditions a linear predictor on the embedding and transformation parameters, while the latter maps these parameters to a hypernetwork that outputs the predictor weights, preventing them from being ignored. We adopt the hypernetwork formulation because it separates transformation-specific actions from what is shared across transformations. Unlike these works, we use no separate invariance term: a single contrastive objective conditioned on the substitution enforces equivariance to substitutions and invariance to re-descriptions.
\section{Preliminaries}
\subsection{Background and notation}
Fix $n \ge 1$ and $0 \le k \le n$, and let $\sV \coloneq \sR^{n-k} \times \sZ^k$ be the space of our decision variables, with $n-k$ continuous variables, and $k$ integer ones. We introduce the notation $[p] \coloneq \{1, \ldots, p\}$ and denote by $\ve_p$ the $p$-th element of the canonical basis of $\sR^n$ or $\sR^m$ depending on the context: $\ve_p = (0, \dots, 0, 1 , 0, \dots, 0)^\top$ where the unique $1$ lies at position $p$.
\begin{definition}[MILP instance]\label{def:milp}
A \emph{mixed integer linear programming instance} over $\sV$ with $m \ge 2$ constraints is defined as a tuple: 
\begin{equation}
M = (\mA, \vb, \vc) \in \Inst_m \coloneq \sR^{m \times n} \times \sR^m \times \sR^n
\end{equation}
representing the following optimization problem, wherein a linear objective $\vx \mapsto \vc^\T\vx$ is minimized under linear inequality constraints ($\mA \vx \le \vb$), with the $k$ last variables being integer valued ($x \in \sV$):
\begin{equation}
  \min_{\vx\in\sV}\ \vc^\T\vx\quad\text{subject to}\quad\mA\vx\le\vb.
\end{equation}
We will refer to $F(M) \coloneq \{\mA \vx \le \vb\} \cap \sV$ as the feasible region of $M$. $M$ is said to be \emph{feasible} if its feasible region is nonempty, and we designate by $\OPT(M) \in \sR \cup \{ - \infty \}$ its optimal value in that case. It is additionally said to be \emph{bounded} if its optimal value is bounded.
\end{definition}
\subsection{Equivalence--preserving data augmentations}

\begin{definition}[Equivalence]\label{def:equiv}
Instances $M=(\mA,\vb,\vc) \in \Inst_m$ and $M'=(\mA',\vb',\vc') \in \Inst_{m'}$ are \emph{equivalent} if there are an affine bijection
$\varphi$ of $\sR^n$ and $\gamma\in\sR$ such that:
\begin{itemize}[nosep]
    \item $\varphi(F(M))=F(M')$, and
    \item $\vc'^\T\varphi(\vx)=\vc^\T\vx+\gamma$ for all $\vx\in F(M)$.
\end{itemize}
In that case, we call the pair $(\varphi,\gamma)$ a \emph{witness} for the equivalence of $M$ and $M'$, and refer to $M'$ as a \emph{reformulation} of $M$.
\end{definition}
Equivalent instances are either both infeasible, both unbounded, or both have bounded optimal solutions. In the latter case, $\varphi$ maps the optimal solutions of $M$ onto those of $M'$, and $\OPT(M')=\OPT(M)+\gamma$.
\begin{definition}[Augmentation rules]\label{def:rule}
An \emph{augmentation rule} of shape $m\to m'$ is a map $\Phi:~\Inst_m\to~\Inst_{m'}$. It is \emph{valid} if $\Phi(M)$ is equivalent to $M$ for every $M\in\Inst_m$, and it is affine if every component of  $\Phi(\mA,\vb,\vc)$ is an affine function of the components of $(\mA,\vb,\vc)$. 
\end{definition}
A data augmentation pipeline acting on instances of all sizes can be viewed as a family of rules, with one rule for each $m \ge 2$. The pipeline is valid if and only if every rule in this family is valid. We can therefore study augmentation rules of a fixed shape without loss of generality.

These definitions come with some restrictions: two instances can't be equivalent if their feasible regions do not live on the same space $\sV$, which means that augmentations that increase or decrease the number of variables in an instance\footnote{An example of such an augmentation would be breaking an integer variable bounded by $N$ into $\left\lceil \log_b(N) \right\rceil$ integer variables using its base $b$ representation.} do not fit into our framework.

\subsection{Bipartite graph encoding}
\label{sec:encoding}
Following~\citet{Gasse2019-nips}, we represent an instance $M=(\mA,\vb,\vc)\in\Inst_m$ as a bipartite graph $\mathcal G=(V\sqcup C,E)$, with one node $i\in V$ per variable and one node $j\in C$ per constraint. Variable $i$ has features $(\evc_i,\1_{i>n-k})$, constraint $j$ has feature $b_j$, and $E=\{(i,j):\emA_{j,i}\neq 0\}$, with edge feature $\emA_{j,i}$. A GNN $f_\vtheta$ maps this graph to node embeddings. We denote by $\vh_i(M)\in\sR^d$ the embedding of variable $i$ and collect the variable embeddings in the matrix
\begin{equation}
    \mH_\vtheta(M) := [\vh_1(M),\ldots,\vh_n(M)]^\top \in\sR^{n\times d}.
\end{equation}

\section{Method}

We learn general-purpose MILP representations from equivalent formulations. We first show that valid affine reformulations decompose into re-descriptions, which leave variables unchanged, and substitutions, which transform them predictably. ReMILP uses this structure as a self-supervised inductive bias, training variable embeddings to be invariant to re-descriptions and equivariant to substitutions. ReMILP realizes this principle with a contrastive objective that, without solver-derived labels, pre-trains an encoder to capture reusable MILP structure before task-specific adaptation.

\subsection{Valid augmentation rules}
Two important families of data augmentation rules satisfying \cref{def:rule} are \emph{substitutions} and \emph{re-descriptions}.  A \emph{substitution} $\Sub(\mT,\vq)$, with $\mT\in\GL_n(\sR)$ and $\vq\in\sR^n$, applies the change of variables
$\vx=\mT\vy+\vq$ to instances with any number of constraints:
\begin{equation}\label{eq:sub}
  \Sub(\mT,\vq):\ (\mA,\vb,\vc)\longmapsto
  \bigl(\mA\mT,\ \vb-\mA\vq,\ \mT^\T\vc\bigr).
\end{equation}
A \emph{constraint rewriting} $\Red(\mLambda,\vg)$, with
$\mLambda\in\sR^{m'\times m}$ and an arbitrary map $\vg:\sR^m\to\sR^{m'}$,
rewrites the constraint system and keeps variables and objective:
\begin{equation}\label{eq:redesc}
  \Red(\mLambda,\vg):\ (\mA,\vb,\vc)\longmapsto
  \bigl(\mLambda\mA,\ \vg(\vb),\ \vc\bigr).
\end{equation}
We define a \emph{re-description} as a constraint rewriting that preserves the
feasible set of every instance. Every re-description is a valid rule and the converse holds for affine $\vg$.

\begin{restatable}[Substitutions]{theorem}{thmsub}\label{thm:sub}
The substitution $\Sub(\mT,\vq)$ is valid if and only if $\vq\in\sV$ and
$$
  \mT=\begin{pmatrix}\mT_{\cblk\cblk}&\mT_{\cblk\iblk}\\
    \bm{0}&\mT_{\iblk\iblk}\end{pmatrix},\qquad
  \mT_{\cblk\cblk}\in\GL_{n-k}(\sR),\quad
  \mT_{\cblk\iblk}\in\sR^{(n-k)\times k},\quad
  \mT_{\iblk\iblk}\in\GL_k(\sZ).
$$
\end{restatable}
As such, the admissible transformations of integer variables are not limited to the shifts and signed permutations studied by \citet{Lin2024-iclr} and \citet{Lu2026-iclr}, but include arbitrary unimodular maps.
\begin{restatable}[Re-descriptions]{theorem}{thmredesc}\label{thm:redesc}
The constraint rewriting $\Red(\mLambda,\vg)$ is a re-description if and only if
$\mLambda\ge0$ and, for every $\vb\in\sR^m$,
\begin{itemize}[nosep]
\item $\vg(\vb)\ge\mLambda\vb$, and
\item for all $p\in[m]$, there are $i \in [m']$ and $\theta>0$ such that $\vlambda_i^\T=\theta\ve_p^\T$ and
  $g_i(\vb)=\theta b_p$.
\end{itemize}
For affine $\vg(\vb)=\mLambda\vb+\vs$ this means that up to row permutation,
$\mLambda=\begin{pmatrix}\mPi\\ \mN\end{pmatrix}$ and  $\vs=\begin{pmatrix}\vzero\\ \vs'\end{pmatrix}$
where $\mPi$ is a positive generalized permutation matrix, $\mN\ge0$ and $\vs'\ge0$.
\end{restatable}
A re-description therefore has to keep a positively rescaled copy of every constraint, and may additionally append nonnegative combinations of constraints with loosened right-hand sides.

A composite of a valid substitution and a re-description is itself a valid augmentation rule:
\begin{equation}\label{eq:composite}
  \sRed(\mLambda,\mLambda\vb+\vs)\circ\Sub(\mT,\vq):\
  (\mA,\vb,\vc)\longmapsto
  \bigl(\mLambda\mA\mT,\ \mLambda\vb+\vs-\mLambda\mA\vq,\ \mT^\T\vc\bigr).
\end{equation}
\begin{restatable}[Completeness]{theorem}{thmcomplete}\label{thm:complete}
An affine rule $\Phi:\Inst_m\to\Inst_{m'}$ is valid if and only
if $\Phi=\sRed(\mLambda,\mLambda\vb+\vs)\circ\Sub(\mT,\vq)$ with $(\mT,\vq)$
as in \cref{thm:sub} and $(\mLambda,\vs)$ as in \cref{thm:redesc}.
\end{restatable}
Thus, Theorems \ref{thm:sub}--\ref{thm:complete} characterize all affine augmentation rules valid for every MILP instance: each decomposes into a re-description and a substitution. Proofs are provided in \cref{app:proofs}. This characterization provides a natural inductive bias for MILP representation learning: knowing how each part of a reformulation acts on the problem tells us how embeddings should respond to it.

\begin{figure}[htbp]
  \centering
  \resizebox{0.80\linewidth}{!}{\input{fig_remilp_views.tex}}
  \caption{\textbf{ReMILP pre-training overview}. \textbf{(1)} Independent re-descriptions (\cref{thm:redesc}) form two views of the same MILP instance $M_0$, with a substitution (\cref{thm:sub}) applied to the second view~\cref{eq:reformulations}. \textbf{(2)} A shared two-layer GAT encodes their bipartite graphs. \textbf{(3)} The hypernetwork uses local substitution parameters and the variable partner embedding to produce a linear operator predicting $\vh_i(M')$ from $\vh_i(M)$~\cref{eq:local_pred}. \textbf{(4)} The pre-training contrastive loss makes this prediction the preferred match to $\vh_i(M')$ over predictions from other descriptors and embeddings of other variables~\cref{eq:loss}.}
  \label{fig:remilp-method}
\end{figure}
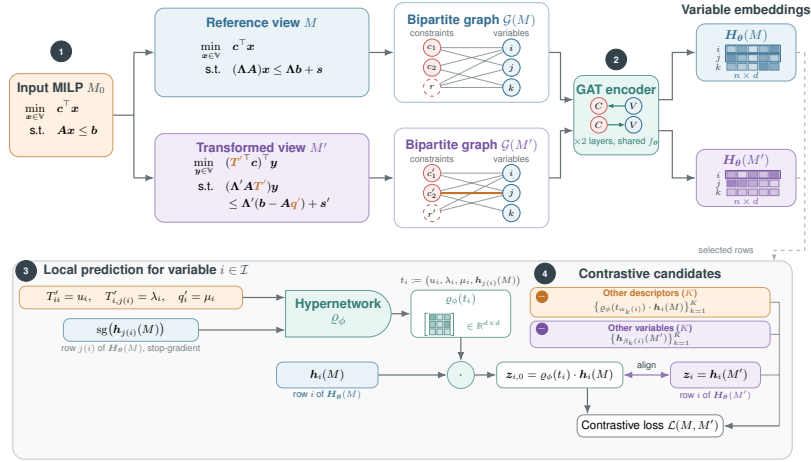
\subsection{Invariance and equivariance}

As introduced in \cref{sec:encoding}, a GNN maps each instance $M$ to its matrix of variable embeddings $\mH_\vtheta(M)$. We use the reformulations characterized above to define an inductive bias on these embeddings. Valid re-descriptions do not modify an instance's feasible region nor its optimal solutions, so we encourage \emph{invariance}:
\begin{equation}
    \mH_\vtheta\left(\sRed(\mLambda,\vg) \cdot M\right) = \mH_\vtheta(M).
\end{equation}
Substitutions, via variable scaling or transvection, instead change the feasible region and optimal solutions in a known way. The variable embeddings should thus be sensitive to substitutions, which we express through \emph{equivariance}:
\begin{equation}
    \mH_\vtheta\left(\Sub(\mT,\vq) \cdot M\right) = \rho(\mT,\vq) \cdot \mH_\vtheta(M),
\end{equation}
where $\rho$  maps each substitution $(\mT,\vq)$ to a linear operator predicting how the variable embeddings transform. Intuitively, $\rho$ can be seen as a representation of the group of valid substitutions, though we require it neither to be a homomorphism nor to output invertible matrices. These transformation laws encourage the representation space to capture MILP-specific structure without solver-derived labels.
\subsection{Reformulation-contrastive learning}

We turn the invariance and equivariance inductive biases into a self-supervised objective using pairs of reformulated MILP instances. \cref{fig:remilp-method} summarizes ReMILP pre-training. Since general substitutions can
densify the MILP graph and make GNN inference prohibitively expensive, we restrict
the substitution matrix $T$ to be sparse.

Concretely, we partition the variables into two sets $[n]=\gI\sqcup\gJ$ and transform only variables in $\gI$. For each $i\in\gI$, we choose a partner $j(i)\in\gJ$ and apply
the local substitution
\begin{equation}
    x_i \leftarrow u_i x_i+\lambda_i x_{j(i)}+\mu_i.
    \label{eq:affine_substitution}
\end{equation}
These local substitutions together define a sparse matrix $\mT'$ and a translation vector $\vq'$: $T'_{ii}=u_i$, $T'_{i,j(i)}=\lambda_i$, and $q'_i=\mu_i$ for $i\in\gI$. For $j\in\gJ$, $T'_{jj}=1$ and $q'_j=0$. All other components of $\mT'$ are zero. We choose the parameters to satisfy the conditions of \cref{thm:sub}. In particular, if $x_i$ is an integer, $u_i\in\{-1,1\}$ and the only nontrivial substitution we apply to a binary variable is $x_i\leftarrow1-x_i$. The way these augmentations are sampled is detailed in \cref{app:aug}.

Given an input instance $M_0$, we construct two views. For the first, we apply a re-description to $M_0$. For the second, we first apply the substitution $\Sub(\mT',\vq')$, then independently sample a re-description of the resulting instance. Each re-description keeps the existing constraints and adds redundant constraints formed from random nonnegative combinations with relaxed right-hand sides. The matrices $\mLambda,\mLambda'$ specify the row combinations, and $\vs,\vs'$ specify the bound relaxations. Thus,
\begin{equation}
\begin{aligned}
    M  &= \sRed(\mLambda,\mLambda\vb+\vs)\cdot M_0,\\
    M' &= \sRed(\mLambda',\mLambda'\vb+\vs')
          \circ\Sub(\mT',\vq')\cdot M_0.
\end{aligned}
\label{eq:reformulations}
\end{equation}

We predict the embedding of variable $i$ in $M'$ from its embedding in $M$. Rather than conditioning on the full substitution, we use its local parameters $(u_i,\lambda_i,\mu_i)$ together with the embedding $\vh_{j(i)}(M)$ of the partner variable involved in the local substitution of $x_i$. Inspired by the hypernetwork approach of \citet{Garrido2023-icml}, we map these local inputs to an affine operator that predicts the transformed embedding of variable $i$:
\begin{equation}
    \vh_i(M') \approx
    \varrho_\phi\!\left(
        u_i,\lambda_i,\mu_i,
        \sg\!\left(\vh_{j(i)}(M)\right)
    \right)\cdot\vh_i(M),
    \label{eq:local_pred}
\end{equation}
where $\sg$ denotes stop-gradient. We write $t_i\coloneq (u_i,\lambda_i,\mu_i,\sg(\vh_{j(i)}(M)))$ for the hypernetwork's input. More details about the hypernetwork are available in \cref{app:hypernetwork}.

We adapt the normalized temperature-scaled cross-entropy (NT-Xent) objective \citep{pmlr-v119-chen20j} to align $\varrho_\phi(t_i)~\cdot~\vh_i(M)$ with $\vh_i(M')$ for each $i\in\gI$. We use two types of negatives: predictions obtained with an incorrect transformation descriptor and embeddings of incorrect variables in the transformed view. For $K\geq1$, we sample indices $\alpha_1(i),\ldots,\alpha_K(i)$ uniformly from $\gI\setminus\{i\}$ and independently sample indices $\beta_1(i),\ldots,\beta_K(i)$ uniformly from $[n]\setminus\{i\}$. Let $\vz_i\coloneq\vh_i(M')$ be the anchor. Its positive and the two sets of negatives are
\begin{equation}
\vz_{i,k}\coloneq
\begin{cases}
\varrho_\phi(t_i) \cdot \vh_i(M)
        & \text{if } k = 0, \\
    \varrho_\phi\left(t_{\alpha_k(i)}\right) \cdot \vh_i(M)
        & \text{if } 1 \le k \le K, \\
    \vh_{\beta_{k-K}(i)}(M')
        & \text{if } K + 1 \le k \le 2K.
\end{cases}
\end{equation}

Denoting cosine similarity by $\operatorname{sim}(\cdot,\cdot)$, our contrastive loss is
\begin{equation}
\mathcal{L}(M,M') = -\frac{1}{|\gI|} \sum_{i \in \gI} \log\left( \frac{
    \exp\left(
        \operatorname{sim}\left(\vz_i,\vz_{i,0}\right)/\tau
    \right)
}{
    \sum_{k=0}^{2K}
    \exp\left(
        \operatorname{sim}\left(\vz_i,\vz_{i,k}\right)/\tau
    \right)
}
\right).
\label{eq:loss}
\end{equation}
By jointly training $f_\vtheta$ and $\varrho_\phi$, this objective encourages variable embeddings to be stable under re-descriptions and to transform predictably under substitutions, as we show empirically in \cref{sec:inv+equiv}.
\section{Empirical evaluation}
We evaluate the representations learned by ReMILP on several classes and difficulties of MILP instances to investigate whether the learned invariance and equivariance to relevant reformulations provide a useful prior over the space of MILPs for downstream tasks. 
\subsection{Experimental setup}
\label{sec:setup}
\paragraph{Baselines}
We compare ReMILP with FORGE\textsuperscript{\textdagger}, our re-implementation of the vector-quantized graph autoencoder of \citet{Shafi2026-tmlr}, and a randomly initialized encoder. All three use the same graph encoder backbone in order to fairly assess the contribution of each pre-training method.
\paragraph{Pre-training protocol} We use 29\,000 MILP-Evolve instances \citep{Li2025-iclr} to pre-train the encoder. Both ReMILP and FORGE\textsuperscript{\textdagger} are pre-trained for 20\,000 steps with a learning rate of $10^{-3}$. The training procedure is described in more detail in \cref{app:training}.
\paragraph{Architecture} The graph encoder uses two layers of four-head GATv2 attention \citep{brody2022how} with 128-dimensional hidden states on the MILP bipartite graph. Each layer updates constraints from variables, then variables from constraints \citep{Gasse2019-nips}. Jumping Knowledge \citep{pmlr-v80-xu18c} combines layer representations.
\paragraph{Downstream tasks}  We evaluate the learned representations on three tasks: integrality-gap prediction at the instance level, binary-solution prediction at the variable level, and constraint-activity prediction at the constraint level. For integrality-gap prediction, we use 12\,500 held-out MILP-Evolve instances \citep{Li2025-iclr}. For the binary solutions and constraint activity tasks, we use ten problem classes from the Distributional MIPLIB (D-MIPLIB) dataset \citep{Huang2024-arxiv}, comprising 23 instance families. We describe the datasets in more detail in \cref{app:datasets}.
\paragraph{Downstream training and tuning}
We use two-layer MLP prediction heads for all downstream tasks. For binary-solution and constraint-activity prediction, the head maps each variable or constraint embedding, respectively, to a node-level prediction. For integrality-gap prediction, it maps an instance embedding obtained by attention pooling over both node types to a scalar. Using task-specific labels, we consider two settings: \emph{frozen}, where only the prediction head is trained, and \emph{fine-tuned}, where the head and graph encoder are trained jointly. Further details are provided in \cref{app:details}.

\paragraph{Metrics} For the integrality-gap prediction task, the target is $\frac{|\mathrm{LP}-\mathrm{IP}|}{ \max(|\mathrm{LP}|,|\mathrm{IP}|) + 10^{-5}}$, where $\mathrm{LP}$ and $\mathrm{IP}$ are the respective optima of the LP relaxation and the MILP. We report test mean absolute error (MAE). For the node-level tasks, we compute a pool of high-quality solutions for each instance. For binary-solution prediction, the target for each binary variable is the weighted fraction of solutions in which it equals one, capturing how often it is selected among good solutions. For constraint-activity prediction, the target is the weighted fraction of solutions in which each constraint is binding. More details are provided in \cref{app:datasets}. We evaluate both tasks using the Kullback-Leibler (KL) divergence between these target probabilities and the predictions. To aggregate across instance families, we divide each model's test KL on an instance family by that of a baseline on the same family and take the geometric mean of these ratios. The baseline is the random encoder in the frozen setting and the supervised encoder trained from scratch in the fine-tuned setting.

\subsection{Results}
The introduction motivates ReMILP as a way to learn general-purpose MILP representations from equivalent formulations. We therefore ask whether the learned representations~(i) retain their invariance and equivariance on unseen classes,~(ii) carry task-relevant information when frozen, and~(iii) provide a useful initialization for fine-tuning.
\subsubsection{Invariance and equivariance on unseen classes}
\label{sec:inv+equiv}
We perform nearest-neighbor retrieval between each D-MIPLIB test instance and one
of its reformulations. For invariance, we apply a re-description to each test instance. For every variable, we use its original embedding as a query, rank all variable embeddings in the re-described instance by cosine similarity, and record the rank of the corresponding variable. For equivariance, we instead apply a substitution. For each substituted variable, we again rank all embeddings in the reformulated instance, but use as the query either its original embedding, corresponding to an invariant prediction, or the embedding predicted by ReMILP’s hypernetwork with input the applied substitution. We record the rank of the corresponding substituted variable. \Cref{tab:probes} reports the average Top-1 retrieval accuracy and mean reciprocal rank (MRR) across all 23 D-MIPLIB test splits.
\begin{table}[H]
  \centering
   \caption{
  \textbf{Invariance and equivariance properties.} Nearest-neighbor retrieval after re-description (invariance) or substitution (equivariance) on D-MIPLIB. For equivariance, “+ hyper” uses the hypernetwork prediction. Other rows use the original embedding. Mean ± sd over 3 seeds.
  }  
  \label{tab:probes}
  \begin{tabular}{@{}lcccc@{}}
    \toprule
    & \multicolumn{2}{c}{Invariance $\uparrow$} & \multicolumn{2}{c}{Equivariance $\uparrow$} \\
    \cmidrule(lr){2-3} \cmidrule(lr){4-5}
    Encoder & Top-1 & MRR & Top-1 & MRR \\
    \midrule
    Random                             & $0.289 \pm 0.009$ & $0.375 \pm 0.011$ & $0.091 \pm 0.011$ & $0.121 \pm 0.013$ \\
    FORGE\textsuperscript{\textdagger} & $0.120 \pm 0.040$ & $0.171 \pm 0.049$ & $0.017 \pm 0.003$ & $0.026 \pm 0.004$ \\
    ReMILP (ours)                             & $0.383 \pm 0.093$ & $0.454 \pm 0.089$ & $0.081 \pm 0.014$ & $0.105 \pm 0.016$ \\
    ReMILP+hyper (ours)            & -- & -- & $\mathbf{0.162 \pm 0.005}$ & $\mathbf{0.276 \pm 0.004}$ \\
    \bottomrule
  \end{tabular}
\end{table}
Under re-description, ReMILP achieves the strongest retrieval performance, indicating that its embeddings are more stable to changes in constraint description. It improves slightly over the random encoder, while FORGE\textsuperscript{\textdagger} performs worse than random. Under substitution, treating the embeddings as invariant performs poorly, whereas the hypernetwork substantially improves retrieval. Together, these results support the intended distinction on unseen classes: re-descriptions should preserve embeddings, while substitutions should transform them predictably.

\subsubsection{Task-relevant information in frozen representations} \label{sec:frozen}

We now ask whether pre-training produces representations that are directly useful for downstream prediction, before any task-specific adaptation of the encoder. We therefore evaluate the encoders in the frozen regime. For the node-level tasks, we report pure ILP and MILP families separately to examine whether this transfer depends on the presence of continuous variables. \Cref{tab:frozen,tab:frozengap} summarize the results.

Both pre-trained encoders carry more task-relevant information than the random encoder baseline, whose frozen embeddings are random projections of each node's two-hop neighborhood. This indicates that self-supervised pre-training extracts structure that is useful for downstream prediction even without adapting the encoder. Beyond this common benefit, ReMILP improves over FORGE\textsuperscript{\textdagger} on both binary solution prediction and integrality-gap prediction, with the largest margin on the latter: ReMILP achieves a 24.5\% lower error. Since FORGE\textsuperscript{\textdagger} uses its own pooler based on codeword histograms, we additionally evaluate its frozen instance-level representations using the same attention pooler as Random and ReMILP. Its performance remains essentially unchanged, which confirms that the observed difference is not an artifact of the pooling mechanism.

The node-level results reveal a further distinction across tasks and problem types. On binary solution prediction, ReMILP reduces the normalized KL divergence relative to FORGE\textsuperscript{\textdagger} by 3.4\% on pure ILPs and 8.9\% on MILP families. In contrast, the two encoders are indistinguishable within one standard deviation on constraint-activity prediction. These results are consistent with the structure of our pre-training objective: constraint embeddings are shaped only indirectly through message passing. The larger gain on MILPs is also consistent with richer substitutions for continuous variables. 

\begin{table}[htbp]
  \centering
  \caption{\textbf{Downstream node-level performance under frozen encoders}. Test KL divergence normalized by the corresponding random encoder result for each instance family and aggregated using the geometric mean. Values below \(1\) indicate improvement over the random encoder ($\downarrow$). }
  \begin{tabular}{@{}llcc@{}}
    \toprule
    Task & Type & ReMILP (ours) & FORGE\textsuperscript{\textdagger} \\
    \midrule
    \multirow{2}{*}{Binary solution} & ILP & $\mathbf{0.887 \pm 0.016}$ & $0.918 \pm 0.007$ \\
     & MILP & $\mathbf{0.688 \pm 0.050}$ & $0.755 \pm 0.012$ \\
    \midrule
    \multirow{2}{*}{Constraint activity} & ILP & $0.871 \pm 0.007$ & $0.882 \pm 0.018$ \\
     & MILP & $0.855 \pm 0.028$ & $0.863 \pm 0.024$ \\
    \bottomrule
  \end{tabular}
  \label{tab:frozen}
\end{table}

\subsubsection{Pre-trained encoders as an initialization} \label{sec:finetuned}

When the encoder is fine-tuned jointly with the head, \cref{fig:ft-nodes,fig:ft-gap} show two effects: faster convergence and, for some problem classes, better final performance. On pure ILP families, the benefit is primarily optimization speed. Both pre-trained encoders outperform training from scratch early in fine-tuning, but all three models are within $1\%$ of each other by the end of the fine-tuning budget. Pre-training thus provides a better starting point on these instances, without improving the final solution reached by supervised training. In contrast, on MILP families the advantage persists throughout the fine-tuning budget on both tasks. This effect is clearest for MILP binary solution prediction: ReMILP reaches the final performance of the encoder trained from scratch after only 6\,400 steps and continues improving beyond it, whereas FORGE\textsuperscript{\textdagger} reaches the same level later, after 13\,000 steps. Constraint activity follows the same pattern, with a smaller final gap.

Integrality-gap prediction provides a strong difference in the fine-tuning regime. ReMILP reaches the final MAE of the encoder trained from scratch after 19\,500 of the 100\,000 fine-tuning steps and ultimately improves on it by $16\%$, whereas FORGE\textsuperscript{\textdagger} converges to a higher value. Moreover, fine-tuning changes ReMILP only marginally relative to its frozen performance ($0.040 \pm 0.002$ to $0.038 \pm 0.001$), while training the encoder from a random initialization provides a substantial improvement over the corresponding frozen encoder. This suggests that much of the information required for integrality-gap prediction is already present in ReMILP's pre-trained representations, with fine-tuning primarily refining rather than constructing it. 

The role of pre-training depends on the downstream regime. On pure ILPs, it mainly accelerates supervised learning. On mixed-integer and instance-level tasks, ReMILP can additionally provide task-relevant structure that remains beneficial after full fine-tuning. 

\begin{figure}[h]
  \centering
  \begin{subfigure}[t]{0.24\linewidth}\centering
    \includegraphics[width=\linewidth]{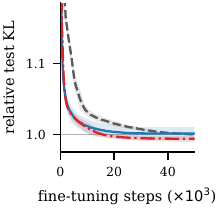}
    \caption{Solution, ILP}\end{subfigure}\hfill
  \begin{subfigure}[t]{0.24\linewidth}\centering
    \includegraphics[width=\linewidth]{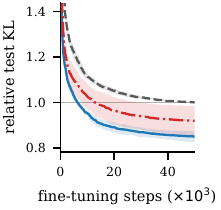}
    \caption{Solution, MILP}\end{subfigure}\hfill
  \begin{subfigure}[t]{0.24\linewidth}\centering
    \includegraphics[width=\linewidth]{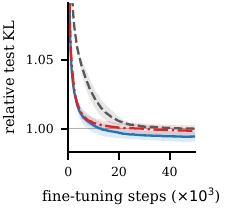}
    \caption{Activity, ILP}\end{subfigure}\hfill
  \begin{subfigure}[t]{0.24\linewidth}\centering
    \includegraphics[width=\linewidth]{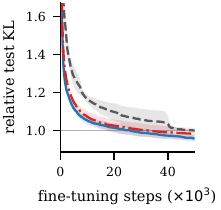}
    \caption{Activity, MILP}\end{subfigure}
  \par\vspace{0.2em}
  \includegraphics[width=0.55\linewidth]{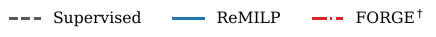}
  \caption{\textbf{Fine-tuning on binary solution and constraint activity prediction.} Test KL throughout fine-tuning, aggregated by geometric mean across the corresponding D-MIPLIB families and normalized by the final performance of the encoder trained from scratch (“Supervised”).}
      \label{fig:ft-nodes}
\end{figure}

\begin{figure}[htbp]
  \centering
  \begin{minipage}[c]{0.43\linewidth}
    \centering\small
    \captionof{table}{\textbf{Integrality gap under frozen encoders}. Test MAE on MILP-Evolve.}
    \label{tab:frozengap}
    \begin{tabular}{@{}lc@{}}
      \toprule
      Encoder & Test MAE $\downarrow$ \\
      \midrule
      Random & $0.063 \pm 0.008$ \\
      ReMILP (ours) & $\mathbf{0.040 \pm 0.002}$ \\
      FORGE\textsuperscript{\textdagger} & $0.053 \pm 0.007$ \\
      FORGE\textsuperscript{ \textdagger} + attn. & $0.054 \pm 0.002$ \\
      \bottomrule
    \end{tabular}
  \end{minipage}\hspace{0.035\linewidth}%
  \begin{minipage}[c]{0.35\linewidth}
    \centering
    \includegraphics[width=\linewidth]{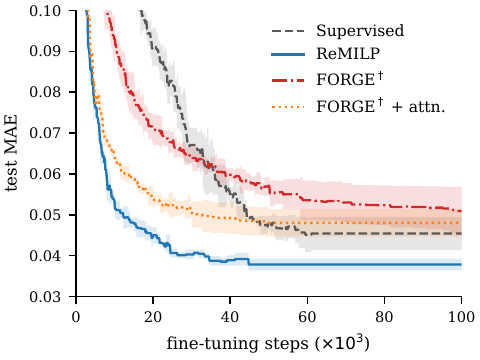}
    \captionof{figure}{\textbf{Fine-tuning on integrality gap.} Test MAE on MILP-Evolve}
    \label{fig:ft-gap}
  \end{minipage}
\end{figure}

\section{Conclusion}

We introduced a theory-driven approach to self-supervised representation learning for MILPs. We characterized all valid affine augmentation rules, showing that they decompose into re-descriptions and substitutions. This characterization extends prior MILP augmentations to richer variable transformations, including integer transvections, and provides a natural inductive bias for variable embeddings: invariance to re-description and equivariance to substitutions. ReMILP instantiates this principle by pre-training on equivalent reformulation pairs. A contrastive objective jointly trains the encoder and hypernetwork to predict how local substitutions transform variable embeddings, without solver-derived labels. Empirically, ReMILP's learned representations exhibit stronger invariance and equivariance than the baselines on unseen classes, retain task-relevant information when frozen, and provide a useful initialisation for fine-tuning. These results suggest that known reformulation behaviors can structure the latent space in ways that benefit several downstream MILP tasks. More broadly, this work points to a promising direction for MILP representation learning, where known problem structure shapes latent representations through equivariant self-supervision.

\clearpage
\subsection*{AI use statement}
In this work, we used generative AI tools in developing the proofs of \cref{thm:sub,thm:redesc,thm:complete}. Several key results were first proven by AI, and then verified and simplified by the authors. LLMs were additionally used in code generation, as well as drafting and editing parts of this paper. We have reviewed all AI-assisted work, and we take responsibility for the final content of this work, including text, claims or artifacts produced with the aid of generative AI.

\subsection*{Reproducibility statement}
We make the processed datasets we use in our experiments available at \url{https://huggingface.co/datasets/orailix/remilp-data}. The code and reproducible Pixi environment are available at
\url{https://github.com/orailix/remilp}. We also detail our implementation in \cref{sec:setup} and \cref{app:details}.

\subsection*{Acknowledgements}
This work received financial support from SNCF through the research chair “AI and optimization for mobility” with École Polytechnique. This work was granted access to the HPC resources of IDRIS under the allocation AD011017427 made by GENCI. Finally, we thank Alexi Canesse, Benoît Goupil, Mohamed Dhouib, Mahammed El Sharkawy, Martin Bonsergent-Brachet, Lucas Thil and Hossein Nick Zinat Matin for helpful discussions and feedback on this work.

\bibliography{iclr2027_conference}
\bibliographystyle{iclr2027_conference}


\clearpage 
\appendix
\section{Omitted proofs}
\label{app:proofs}

Throughout this section, $n$, $k$, $m'$ and $m\ge 2$ are fixed. Recall that $\ve_j$ is used to denote the $j$-th vector of the canonical basis of $\sR^n, \sR^m, \sR^k$ or $\sR^{n-k}$ depending on the context.

The last $k$ of any $n$ coordinates are the integer ones, and a subscript $\cblk$ or $\iblk$ selects the corresponding parts of a vector or matrix. For example, a vector $\vx\in\sR^n$ is decomposed in the following manner: $\vx=(\cp{\vx},\ip{\vx})$ with $\cp{\vx}\in\sR^{n-k}$ and $\ip{\vx}\in\sR^k$. For a matrix $\mA\in\sR^{m\times n}$, we decompose it as $\mA=[\cp{\mA}\ \ip{\mA}]$ where $\cp{\mA}\in\sR^{m\times(n-k)}$ and $\ip{\mA}\in\sR^{m\times k}$, so that $\mA\vx=\cp{\mA}\cp{\vx}+\ip{\mA}\ip{\vx}$. Similarly, an $n \times n$ matrix $\mT$ can be decomposed into four blocks:
$$
\bordermatrix{
         & n-k & k \cr
    n-k & \mT_{\cblk\cblk} & \mT_{\cblk\iblk} \cr
    k & \mT_{\iblk\cblk} & \mT_{\iblk\iblk} \cr
}
$$
With the block notation established, we start off by proving some key lemmas.
\subsection{Preliminaries}
\begin{lemma}[Affine self-maps of $\sV$]\label{lem:mixed}
Let $\psi(\vx)=\mS\vx+\vr$ be an affine map of $\sR^n$ with $\psi(\sV)=\sV$. Then $\mS_{\iblk\cblk}=\bm{0}$,
$\mS_{\cblk\cblk}\in\GL_{n-k}(\sR)$, $\mS_{\iblk\iblk}\in\GL_k(\sZ)$ and
$\vr\in\sV$. In particular $\psi$ is a bijection of $\sR^n$. Conversely, every $\psi$ of this
form satisfies $\psi(\sV)=\sV$, and so does $\psi^{-1}$.
\end{lemma}
\begin{proof}
For every $\cp{\vx}\in\sR^{n-k}$ the integer block of $\psi(\cp{\vx},\vzero)$,
namely $\mS_{\iblk\cblk}\cp{\vx}+\ip{\vr}$, lies in $\sZ^k$. A continuous map
from $\sR^{n-k}$ to $\sZ^k$ is constant,
so $\mS_{\iblk\cblk}=\bm{0}$ and $\ip{\vr}\in\sZ^k$. The integer block of
$\psi(\vzero,\ve_j)$ is $\mS_{\iblk\iblk}\ve_j+\ip{\vr}\in\sZ^k$, so
$\mS_{\iblk\iblk}$ is an integer matrix. For every $\vw\in\sV$ there is
$\vx\in\sV$ with $\psi(\vx)=\vw$, and then
$\ip{\vw}=\mS_{\iblk\iblk}\ip{\vx}+\ip{\vr}$. Hence
$\mS_{\iblk\iblk}\sZ^k=\sZ^k$, which proves that $\mS_{\iblk\iblk}$ is unimodular.
Fixing $\ip{\vw}$, which fixes
$\ip{\vx}=\mS_{\iblk\iblk}^{-1}(\ip{\vw}-\ip{\vr})$, and letting $\cp{\vw}$
range over $\sR^{n-k}$ in
$\cp{\vw}=\mS_{\cblk\cblk}\cp{\vx}+\mS_{\cblk\iblk}\ip{\vx}+\cp{\vr}$ shows
that $\mS_{\cblk\cblk}$ is surjective, hence invertible. Thus
$\det\mS=\det\mS_{\cblk\cblk}\det\mS_{\iblk\iblk}\neq0$ and $\vr\in\sV$.

Conversely, if $\mS$ and $\vr$ have the stated form, the integer block of
$\psi(\vx)$ is $\mS_{\iblk\iblk}\ip{\vx}+\ip{\vr}\in\sZ^k$ for $\vx\in\sV$, so
$\psi(\sV)\subseteq\sV$. The inverse $\psi^{-1}(\vy)=\mS^{-1}\vy-\mS^{-1}\vr$
has
$\mS^{-1}=\bmat{\mS_{\cblk\cblk}^{-1}}{-\mS_{\cblk\cblk}^{-1}\mS_{\cblk\iblk}\mS_{\iblk\iblk}^{-1}}{\bm{0}}{\mS_{\iblk\iblk}^{-1}}$,
of the same form because $\mS_{\iblk\iblk}^{-1}$ is again unimodular, and
$\mS^{-1}\vr\in\sV$. Hence also
$\psi^{-1}(\sV)\subseteq\sV$, and $\psi(\sV)=\sV$.
\end{proof}
\begin{lemma}[Objective rigidity]\label{lem:rigidity}
Let $(\varphi,\gamma)$ witness the equivalence of $M=(\mA,\vb,\vc)$ and
$M'=(\mA',\vb',\vc')$, with $\varphi(\vx)=\mS\vx+\vr$, and suppose that
$F(M)$ affinely spans $\sR^n$. Then $\mS^\T\vc'=\vc$. In particular, if
$\vc'=\vc$, then $\vc^\T\varphi(\vx)=\vc^\T(\vx+\vr)$ for all
$\vx\in\sR^n$.
\end{lemma}
\begin{proof}
The affine functions $\vx\mapsto\vc'^\T(\mS\vx+\vr)$ and
$\vx\mapsto\vc^\T\vx+\gamma$ agree on $F(M)$, hence everywhere, since $F(M)$ affinely spans $\sR^n$.
In particular, $\mS^\T\vc'=\vc$. If
$\vc'=\vc$ then $\vc^\T\mS\vx=\vc^\T\vx$ for all $\vx$.
\end{proof}
\begin{lemma}[Unconstrained feasible sets]\label{lem:free}
Let $\varphi$ be an affine bijection of $\sR^n$, and let
$\mA'\in\sR^{m'\times n}$, $\vb'\in\sR^{m'}$ be such that $\mA'\vy\le\vb'$ for
every $\vy\in\varphi(\sV)$. Then $\mA'=\bm{0}$ and $\vb'\ge0$. In particular, if an instance
$M'$ is equivalent to an instance whose feasible set is $\sV$, then
$F(M')=\sV$.
\end{lemma}
\begin{proof}
The polyhedron $P'=\{\vy \in \sR^n \mid \mA'\vy\le\vb'\}$ is convex and contains
$\varphi(\sV)$, so it contains
$\conv\varphi(\sV)=\varphi(\conv\sV)=\varphi(\sR^n)=\sR^n$. A nonzero row of $\mA'$ would make $P'$
a proper subset of $\sR^n$, so $\mA'=\bm{0}$, and then $\vzero\le\vb'$. For the
second statement,
$F(M')=P'\cap\sV=\sV$.
\end{proof}
\subsection{Proof of \texorpdfstring{\cref{thm:sub}}{Theorem 4.1}}
\label{app:sub}
\thmsub*

\begin{proof}
Write $\psi(\vx):=\mT^{-1}(\vx-\vq)$ and $P:=\{\vx\in\sR^n \mid \mA\vx\le\vb\}$.
For every instance, substituting $\vx=\mT\vy+\vq$ gives
\begin{equation}\label{eq:sub-feasible}
  F\bigl(\Sub(\mT,\vq)\cdot M\bigr)=\{\vy\in\sV \mid \mA(\mT\vy+\vq)\le\vb\}
  =\psi(P)\cap\sV.
\end{equation}

\paragraph{Sufficiency.} By \cref{lem:mixed}, the map $\vy\mapsto\mT\vy+\vq$ sends
$\sV$ onto $\sV$, and so does its inverse $\psi$. Since $\psi$ is a bijection,
$\psi(F(M))=\psi(P\cap\sV)=\psi(P)\cap\psi(\sV)=\psi(P)\cap\sV$, which is
$F(\Sub(\mT,\vq)\cdot M)$ by \cref{eq:sub-feasible}. As for the objective value map, we have the following identity:
$$(\mT^\T\vc)^\T\psi(\vx)=\vc^\T(\vx-\vq)=\vc^\T\vx-\vc^\T\vq$$

\paragraph{Necessity of the block form of $\mT$.} For $\vc\in\sR^n$ the unconstrained
instance $M_{\vc}=(\bm{0},\vzero,\vc)$ and its image
$(\bm{0},\vzero,\mT^\T\vc)$ both have feasible set $\sV$. Let
$(\varphi_{\vc},\gamma_{\vc})$ be a witness and $\mS_{\vc}$ the linear part of
$\varphi_{\vc}$. Then $\varphi_{\vc}(\sV)=\sV$, so $\mS_{\vc}$ is block upper
triangular with a unimodular integer block by \cref{lem:mixed}, and
$\mS_{\vc}^\T\mT^\T\vc=\vc$ by \cref{lem:rigidity}. As $\vc$ ranges over
$\sR^n$ so does $\vd:=\mT^\T\vc$; hence for every $\vd\in\sR^n$ there is a block
upper triangular $\mS$ with unimodular integer block such that
$\mS^\T\vd=\mT^{-\T}\vd$. For such $\mS$ and $\vd=(\vzero,\ip{\vd})$ we have
$\mS^\T\vd=(\vzero,\mS_{\iblk\iblk}^\T\ip{\vd})$. Writing
$\mT^{-\T}=\bmat{\mX}{\mY}{\mZ}{\mW}$, we obtain $\mY\ip{\vd}=\vzero$ for all
$\ip{\vd}$, so $\mY=\bm{0}$, and
\[
  \sR^k=\bigcup_{\mU\in\GL_k(\sZ)}\ker\bigl(\mW-\mU^\T\bigr).
\]
A proper linear subspace has Lebesgue measure zero, so a countable union of
them is not $\sR^k$ and one of these kernels is all of $\sR^k$, that is, $\mW=\mU^\T$ for some $\mU\in\GL_k(\sZ)$. Therefore
$\mT^{-1}=\bmat{\mX^\T}{\mZ^\T}{\bm{0}}{\mU}$, with $\mX$ invertible because
$\mT$ is, and by \cref{lem:mixed} its inverse $\mT$ has the same form, with
integer block $\mU^{-1}\in\GL_k(\sZ)$.

\paragraph{Necessity of $\vq\in\sV$.} Let $j$ be an integer coordinate and
$\rho\in(0,1)$. Consider the instance whose first two constraints are
$x_j\le\rho$ and $-x_j\le\rho$, with all further constraints (if $m>2$) equal
to $\vzero^\T\vx\le0$. Its feasible set contains $\vzero$. By
\cref{eq:sub-feasible} its image has feasible set
$\{\vy\in\sV \mid |(\mT\vy)_j+q_j|\le\rho\}$. By the block form, $(\mT\vy)_j$ is an
integer combination of the entries of $\ip{\vy}$, hence an integer. If
$q_j\notin\sZ$ and $\rho<\dist(q_j,\sZ)$, the image is infeasible, and no
bijection maps a nonempty set onto the empty set. Hence $\ip{\vq}\in\sZ^k$,
that is, $\vq\in\sV$.
\end{proof}

\subsection{Proof of \texorpdfstring{\cref{thm:redesc}}{Theorem 4.2}}
\label{app:redesc}

We start off by proving a different necessary and sufficient condition from the one stated in \cref{thm:redesc}, which is expressed in terms of the existence of a map from the new constraint system back to the old one, and then we establish the equivalence of the two.

\begin{lemma}\label{lem:redesc-cert}
The constraint rewriting $\Red(\mLambda,\vg)$ is a re-description if and
only if $\mLambda\ge0$ and,
for every $\vb\in\sR^m$, $\vg(\vb)\ge\mLambda\vb$ and there is
$\mGam_{\vb}\in\sR^{m\times m'}$ with $\mGam_{\vb}\ge0$,
$\mGam_{\vb}\mLambda=\mI_m$ and $\mGam_{\vb}\vg(\vb)=\vb$.
\end{lemma}

\begin{proof}
Write $F=\{\vx\in\sV \mid \mA\vx\le\vb\}$ and
$F'=\{\vx\in\sV \mid \mLambda\mA\vx\le\vg(\vb)\}$.

\paragraph{Sufficiency.} Fix $(\mA,\vb)$. If $\vx\in F$, then
$\mLambda\mA\vx\le\mLambda\vb\le\vg(\vb)$ because $\mLambda\ge0$, so
$\vx\in F'$. If $\vx\in F'$, then
$\mA\vx=\mGam_{\vb}\mLambda\mA\vx\le\mGam_{\vb}\vg(\vb)=\vb$ because
$\mGam_{\vb}\ge0$, so $\vx\in F$.

\paragraph{Necessity.} Assume $\Red(\mLambda,\vg)$ is a re-description,
so that $F=F'$ for every $(\mA,\vb)$. Each counterexample below exhibits an
instance on which the point $\ve_1\in\sV$ lies in exactly one of $F$ and
$F'$. All of them use a constraint matrix of the form $\mA=\vu\,\ve_1^\T$,
the matrix whose first column is $\vu$ and whose other columns vanish.
Notice that such a matrix satisfies $\mA\ve_1=\vu$.

\subparagraph{Nonnegativity of $\mLambda$.} Suppose $\Lambda_{i,l}<0$. Let $\vb=\vzero$ and
$\vu=-N\ve_l$ with $N>0$. Then $F=\{\vx\in\sV \mid x_1\ge0\}\ni\ve_1$, while row
$i$ of the new system at $\ve_1$ reads $|\Lambda_{i,l}|N\le g_i(\vzero)$,
which fails for large $N$.

\subparagraph{The bound $\vg(\vb)\ge\mLambda\vb$.} Suppose $g_i(\vb)<\vlambda_i^\T\vb$ and let
$\vu=\vb$. Every original constraint holds with equality at $\ve_1$, so
$\ve_1\in F$, but row $i$ of the new constraint system evaluated at $\ve_1$ reads
$\vlambda_i^\T\vb\le g_i(\vb)$, which is a contradiction.

\subparagraph{Existence of $\mGam_{\vb}$.} Fix $\vb$ and let $\gZ=\{i \mid
g_i(\vb)=\vlambda_i^\T\vb\}$ be the rows on which $\vg(\vb)$ is tight. Row
$p$ of the condition $\mGam_\vb \mLambda = \mI_m$ and $\mGam_{\vb}\vg(\vb)=\vb$ holds iff there exists a $\vgamma\ge0$ such that $\mLambda^\T\vgamma=\ve_p$ and $\vgamma^\T\vg(\vb)=b_p$. 
For any $\vgamma$ with the first two properties, $\vgamma^\T\mLambda\vb
=(\mLambda^\T\vgamma)^\T\vb=b_p$, so that
$$
  \vgamma^\T\vg(\vb)-b_p=\vgamma^\T\bigl(\vg(\vb)-\mLambda\vb\bigr) =\sum_i\gamma_i\bigl(g_i(\vb)-\vlambda_i^\T\vb\bigr)
$$
This is a sum of nonnegative terms. It vanishes if and only if $\gamma_i=0$ for every
$i\notin\gZ$. Row $p$ of the condition is therefore equivalent to the existence of $\vgamma\ge0$
supported on $\gZ$ with $\mLambda^\T\vgamma=\ve_p$.

Suppose that for
some $p\in[m]$ there is no such $\vgamma$. By Farkas' lemma, there is $\vy\in\sR^m$ with
$\vlambda_i^\T\vy\le0$ for all $i\in\gZ$ and $y_p>0$. Let
$\vu=\vb+\varepsilon\vy$ with $\varepsilon>0$. The original constraint $p$ evaluated at
$\ve_1$ reads $b_p+\varepsilon y_p\le b_p$, which does not hold, so $\ve_1\notin F$.
On the other hand, the new row $i$ evaluated at $\ve_1$ reads
$\vlambda_i^\T\vb+\varepsilon\vlambda_i^\T\vy\le g_i(\vb)$. This holds for
$i\in\gZ$ because there $g_i(\vb)=\vlambda_i^\T\vb$ and
$\vlambda_i^\T\vy\le0$, and for $i\notin\gZ$ once $\varepsilon$ is small
because $g_i(\vb)>\vlambda_i^\T\vb$. So $\ve_1\in F'\setminus F$, a
contradiction. Stacking the vectors $\vgamma$ obtained for $p=1,\dots,m$ as
rows gives $\mGam_{\vb}$.
\end{proof}

\thmredesc*
\begin{proof}
By \cref{lem:redesc-cert} it suffices to show that, for $\mLambda\ge0$ and
$\vg(\vb)\ge\mLambda\vb$, a certificate $\mGam_{\vb}$ exists if and only if
every constraint has a tight positive copy.

If row $i$ has $\vlambda_i^\T=\theta\ve_p^\T$ with $\theta>0$ and
$g_i(\vb)=\theta b_p$, put $1/\theta$ in position $(p,i)$ of $\mGam_{\vb}$. Doing this for every $p$ and filling the rest with zeros gives
$\mGam_{\vb}\ge0$ with $\mGam_{\vb}\mLambda=\mI$ and
$\mGam_{\vb}\vg(\vb)=\vb$.

Conversely, row $p$ of $\mGam_{\vb}$ is a vector $\vgamma\ge0$ with
$\ve_p^\T=\sum_i\gamma_i\vlambda_i^\T$ and $\vgamma^\T\vg(\vb)=b_p$. As in
the proof of \cref{lem:redesc-cert}, the latter says that $\vgamma$ is
supported on the rows with $g_i(\vb)=\vlambda_i^\T\vb$. For every coordinate $l\neq p$, the sum
$\sum_i\gamma_i\Lambda_{i,l}=0$ has nonnegative terms, so
$\gamma_i\Lambda_{i,l}=0$ for all $i$. Coordinate $p$ gives
$\sum_i\gamma_i\Lambda_{i,p}=1$, so some $i$ has $\gamma_i\Lambda_{i,p}>0$.
For that $i$, $\gamma_i>0$ forces $\Lambda_{i,l}=0$ for all $l\neq p$, so
$\vlambda_i^\T=\Lambda_{i,p}\ve_p^\T$ with $\Lambda_{i,p}>0$, and
$g_i(\vb)=\vlambda_i^\T\vb$ because $\vgamma$ is supported on such rows.

For affine $\vg(\vb)=\mLambda\vb+\vs$, the first condition is equivalent to
$\vs\ge0$, and a \emph{copy row} of $p$, that is, a row $i$ with
$\vlambda_i^\T=\theta\ve_p^\T$ and $\theta>0$, is tight if and only if
$s_i=0$, independently of $\vb$. Choosing one tight copy for each $p$ gives $m$ distinct rows, which after permuting output rows form $\mD\mPi$ with $\vs=\vzero$ on them. The
remaining rows form $\mN\ge0$ with $\vs'\ge0$. The converse is immediate.
\end{proof}

\subsection{Proof of \texorpdfstring{\cref{thm:complete}}{Theorem 4.3}}

\thmcomplete*

\paragraph{Sufficiency.}
Let $(\mT,\vq)$ and $(\mLambda,\vs)$ satisfy the conditions of \cref{thm:sub,thm:redesc}. The composition is affine. Next, observe that the pair $(\psi,-\vc^\T\vq)$ with $\psi(\vx)=\mT^{-1}(\vx-\vq)$ witnesses the equivalence of $M$ and $\Sub(\mT,\vq)\cdot M$, and that the identity witnesses the equivalence of $\Sub(\mT,\vq)\cdot M$ and $\Phi(M)$. Composing the two, $(\psi,-\vc^\T\vq)$ is a witness on every instance.

\paragraph{Necessity.}
Let $\Phi$ be affine and valid, and write
$\Phi(\mA,\vb,\vc)=(\mA',\vb',\vc')$, $F=F(\mA,\vb,\vc)$ and
$F'=F(\Phi(\mA,\vb,\vc))$. A composition of valid rules is
valid, since witnesses compose. Since $\Phi$ is affine, there are linear
maps $\Lmap,\Lmap_{\vb},\Lmap_{\vc},\hmap,\mK_{\mA}$, matrices
$\mG,\mG_{\vc},\mK,\mK_{\vb}$ and constants $\mA_0,\vs,\vc_0$ with
\begin{equation}\label{eq:general}
\begin{alignedat}{5}
  \mA' &= \mA'(\mA, \vb, \vc) &&= \Lmap(\mA)     &&+ \Lmap_{\vb}(\vb) &&+ \Lmap_{\vc}(\vc) &&+ \mA_0,\\
  \vb' &= \vb'(\mA, \vb, \vc) &&= \hmap(\mA)     &&+ \mG\vb           &&+ \mG_{\vc}\vc     &&+ \vs,\\
  \vc' &= \vc'(\mA, \vb, \vc) &&= \mK_{\mA}(\mA) &&+ \mK_{\vb}\vb     &&+ \mK\vc           &&+ \vc_0 .
\end{alignedat}
\end{equation}
The proof eliminates or pins down these terms one at a time until only the
form of \cref{thm:complete} is left.
\begin{lemma}\label{lem:shape}
In \cref{eq:general}, $\vc_0=\vzero$, $\mK_{\mA}=0$, $\mK_{\vb}=\bm{0}$ and
$\mK=\mT^\T$ for a matrix $\mT$ of the block form of \cref{thm:sub}. If
moreover $\vc'=\vc$ on every instance, then $\mA_0=\bm{0}$, $\Lmap_{\vb}=0$,
$\Lmap_{\vc}=0$, $\mG_{\vc}=\bm{0}$, $\mG\ge0$ and $\vs\ge0$.
\end{lemma}

\begin{proof}
The instance $(\bm{0},\vzero,\vc)$ has $F=\sV$, so $F'=\sV$ by
\cref{lem:free}. A witness maps $\sV$ onto $\sV$, so its linear part
$\mS_{\vc}$ has the block form by \cref{lem:mixed}, and
$\mS_{\vc}^\T\vc'=\vc$ by \cref{lem:rigidity}. On this instance
$\vc'=\mK\vc+\vc_0$, so $\mS_{\vc}^\T(\mK\vc+\vc_0)=\vc$. At $\vc=\vzero$
this gives $\vc_0=\vzero$, and then $\mS_{\vc}^\T\mK\vc=\vc$ for all $\vc$,
so $\mK$ is invertible. For $\mT:=\mK^\T$ we can apply the argument from
the necessity half of \cref{app:sub} verbatim in order to prove that $\mT$ has the block form of \cref{thm:sub}.

Next let $\vc=\vzero$ and $(\mA,\vb)$ in the open set
\[
  \Omega:=\left\{(\mA,\vb)\in\sR^{m\times n}\times\sR^m \relmiddle|
  \vb>\1,\ \max_{p,j}|A_{p,j}|<1\right\}
\]
Then $F$ contains $\vzero,\ve_1,\dots,\ve_n$ and affinely spans $\sR^n$, so
\cref{lem:rigidity} gives $\vc'=\vzero$, that is,
$\mK_{\mA}(\mA)+\mK_{\vb}\vb=\vzero$ on $\Omega$. A linear map vanishing on
an open set vanishes everywhere, so $\mK_{\mA}=0$ and $\mK_{\vb}=\bm{0}$.

Now consider the case $\vc'=\vc$. For $\vb\ge0$ the instance $(\bm{0},\vb,\vc)$ has
$F=\sV$, so \cref{lem:free} gives $\mA'(\bm{0},\vb,\vc)=\bm{0}$ and
$\vb'(\bm{0},\vb,\vc)\ge0$. The affine map $\mA'$ vanishes on
$\{\bm{0}\}\times\sR^m_{\ge0}\times\sR^n$, hence on its affine span
$\{\bm{0}\}\times\sR^m\times\sR^n$, that is,
$\Lmap_{\vb}(\vb)+\Lmap_{\vc}(\vc)+\mA_0=\bm{0}$ for all $\vb$ and $\vc$, so
$\mA_0=\bm{0}$, $\Lmap_{\vb}=0$ and $\Lmap_{\vc}=0$. Likewise
$\mG\vb+\mG_{\vc}\vc+\vs\ge0$ for all $\vb\ge0$ and all $\vc$. At
$\vb=\vzero$, scaling $\vc$ gives $\mG_{\vc}=\bm{0}$ and then $\vs\ge0$.
Taking $\vb=N\ve_p$ with $N\to\infty$ gives $\mG\ve_p\ge0$, so $\mG\ge0$.
\end{proof}

By \cref{lem:mixed,thm:sub}, $\Sub(\mT^{-1},\vzero)=\Sub(\mT,\vzero)^{-1}$
is valid. We replace $\Phi$ by $\Phi\circ\Sub(\mT^{-1},\vzero)$,
keeping the notation of \cref{eq:general} for the new rule. Its objective
component is $\mK\mT^{-\T}\vc=\vc$, so from now on $\vc'=\vc$, and
\cref{lem:shape} gives
\begin{equation}\label{eq:form1}
  \Phi(\mA,\vb,\vc)=\bigl(\Lmap(\mA),\ \hmap(\mA)+\mG\vb+\vs,\ \vc\bigr),
  \qquad \mG\ge0,\ \vs\ge0 .
\end{equation}
In particular $F'$ depends only on $(\mA,\vb)$.

\begin{lemma}\label{lem:normals}
There is $\mLambda\ge0$ with $\Lmap(\mA)=\mLambda\mA$ for all $\mA$.
\end{lemma}

\begin{proof}
Fix $p \in [m]$, $\va\neq\vzero$ and $\beta\in\sR$, and let $\mA=\ve_p\va^\T$ and
$\vb=\beta\ve_p$, so that $F=\{\vx\in\sV \mid \va^\T\vx\le\beta\}$. Then
$\conv F$ contains the half-space
$H=\{\vx \mid \va^\T\vx\le\beta-\|\va\|_1\}$. Indeed, for any $\vx\in H$, every vertex $\vz$ of the unit cube $\lfloor\vx\rfloor+[0,1]^n$ is an integer vector with $\va^\T\vz\le\va^\T\vx+\|\va\|_1\le\beta$, so it lies in $F$. In particular
$F$ affinely spans $\sR^n$. Let $(\varphi,\gamma)$ be a witness of
$(\mA,\vb,\va)\sim\Phi(\mA,\vb,\va)$, whose objectives are both $\va$, and
let $\vr$ be the translation part of $\varphi$. By \cref{lem:rigidity},
$\va^\T\varphi(\vx)=\va^\T(\vx+\vr)$ for all $\vx\in\sR^n$, and comparing
with the value identity $\va^\T\varphi(\vx)=\va^\T\vx+\gamma$ on $F$ gives
$\gamma=\va^\T\vr$. Hence $\varphi(H)$ is the half-space
$H':=\{\vy \mid \va^\T\vy\le\beta-\|\va\|_1+\gamma\}$. As $\varphi$ is
affine, $\conv F'=\varphi(\conv F)\supseteq H'$.

Write $\va_i'^\T$ for row
$i$ of $\Lmap(\mA)$ and $b'_i$ for component $i$ of $\vb'$. The constraint
$\va_i'^\T\vy\le b'_i$ of the image holds on $F'$, hence on $\conv F'$,
thus on $H'$. We claim that $\va'_i=\mu_i\va$ for some $\mu_i\ge0$. Let $\vv\in\sR^n$
with $\va^\T\vv\le0$. For $\vy\in H'$ and $N\in\mathbb N$ the point
$\vy+N\vv$ lies in $H'$, so $\va_i'^\T\vy+N\va_i'^\T\vv\le b'_i$ for all
$N$, which forces $\va_i'^\T\vv\le0$. Applying this to $\vv$ and $-\vv$ for
every $\vv\perp\va$ gives $\va'_i\perp\va^\perp$, so $\va'_i=\mu_i\va$ for
some $\mu_i\in\sR$, and applying it to $\vv=-\va$ gives $\mu_i\ge0$.

The map $\va\mapsto\va'_i$ is linear, since $\Lmap$ is, and by the claim it
sends every nonzero vector to a nonnegative multiple of itself. Such a map
is $\Lambda_{i,p}\mI$ for a scalar $\Lambda_{i,p}\ge0$: if $n=1$ every
linear map of $\sR$ is a scalar, which the claim makes nonnegative, and if
$n\ge2$, for linearly
independent $\va,\vu$ with multipliers $\mu,\nu$, the image
$\mu\va+\nu\vu$ of $\va+\vu$ is a multiple of $\va+\vu$ only if
$\mu=\nu$. So row $i$ of $\Lmap(\ve_p\va^\T)$ is $\Lambda_{i,p}\va^\T$ for
every $\va$. As $\mA=\sum_p\ve_p\va_p^\T$, row $i$ of $\Lmap(\mA)$ is
$\sum_p\Lambda_{i,p}\va_p^\T=\vlambda_i^\T\mA$.
\end{proof}

By \cref{lem:normals}, \cref{eq:form1} becomes
\begin{equation}\label{eq:form2}
  \Phi(\mA,\vb,\vc)=\bigl(\mLambda\mA,\ \hmap(\mA)+\mG\vb+\vs,\ \vc\bigr),
  \qquad \mLambda\ge0 .
\end{equation}

\paragraph{Hyperplane instances.}
Let $\gT\subseteq\sR^m$ be the set of vectors with no zero entry and entries
of both signs. It is open since its complement is the union of closed sets, and nonempty since $m\ge2$. Fix $\vt\in\gT$ and
consider, for $\xi\in\sR$ and $\va\in\sR^n$, the instance
$(\vt\,\va^\T,\ \xi\vt)$. Its row $p$ reads $t_p\va^\T\vx\le\xi t_p$, and
both signs of $t_p$ occur, so $F=\{\vx\in\sV \mid \va^\T\vx=\xi\}$ is a
hyperplane intersected with $\sV$. By \cref{eq:form2}, the image has constraint matrix
$\mLambda\vt\va^\T$ and right-hand side $\hmap(\vt\va^\T)+\xi\mG\vt+\vs$, so
its row $i$ reads $\ell_i\,\va^\T\vy\le\rho_i(\xi,\va)$ with
\[
  \ell_i:=\vlambda_i^\T\vt,\qquad
  \rho_i(\xi,\va):=\xi\,\vg_i^\T\vt+h_i(\vt\,\va^\T)+s_i,
\]
where $\vg_i^\T$ is the $i$-th row of $\mG$ and $h_i$ the $i$-th entry of
$\hmap$. Each $\rho_i$ is an affine function of $(\xi,\va)$, and
$F'=\{\vy\in\sV \mid \ell_i\,\va^\T\vy\le\rho_i(\xi,\va)\ \forall i\}$. Let
$\Gamma_{\va}:=\va^\T\sV$, which contains $a_j\sZ$ for every $j$.

\begin{lemma}\label{lem:hyperplane}
Let $\vt\in\gT$. There are $\kappa(\vt)\in\sR$ and $\vq(\vt)\in\sR^n$ such that
\[
  \mG\vt=\kappa(\vt)\,\mLambda\vt\qquad\text{and}\qquad
  \hmap(\vt\,\va^\T)=-\mLambda\vt\;\vq(\vt)^\T\va\quad\text{for all }\va\in\sR^n .
\]
Moreover, some row $i$ has $\ell_i>0$ and $s_i=0$, and some row has
$\ell_i<0$.
\end{lemma}

\begin{proof}
Let $\va\neq\vzero$ and $\xi\in\Gamma_{\va}$, so that $F\neq\emptyset$ and
$\va^\T F=\{\xi\}$. A witness $(\varphi,\gamma)$ of
$(\vt\va^\T,\xi\vt,\va)\sim\Phi(\vt\va^\T,\xi\vt,\va)$ satisfies
$\va^\T\varphi(\vx)=\va^\T\vx+\gamma$ on $F$ and maps $F$ onto $F'$, so
$\va^\T F'=\{\xi+\gamma\}$ is a single point. Suppose $\ell_i\le0$ for all
$i$. Pick $\vy\in F'$, an index $j$ with $a_j\neq0$, and $\vv:=\operatorname{sign}(a_j) \cdot \ve_j$ so that $\va^\T\vv=|a_j|$ is satisfied. Then $\vy+N\vv\in F'$ for every
$N\in\mathbb N$, and $\va^\T(\vy+N\vv)=\va^\T\vy+N|a_j|$ takes infinitely
many values, a contradiction. So some $\ell_i>0$, and choosing
$\vv = -\operatorname{sign}(a_j) \cdot \ve_j$ instead shows that some $\ell_i<0$. Hence
\[
  H_+(\xi,\va):=\min_{\ell_i>0}\frac{\rho_i(\xi,\va)}{\ell_i},\qquad
  H_-(\xi,\va):=\max_{\ell_i<0}\frac{\rho_i(\xi,\va)}{\ell_i}
\]
are finite, $H_+$ is concave, $H_-$ is convex, and
$F'=\{\vy\in\sV \mid H_-\le\va^\T\vy\le H_+\}$ provided $\rho_i\ge0$ for the
rows with $\ell_i=0$, and $F'=\emptyset$ otherwise. Since $F'\neq\emptyset$,
$\rho_i\ge0$ for $\ell_i=0$ and $\Delta:=H_+-H_-\ge0$. Since
$\va^\T F'=\Gamma_{\va}\cap[H_-,H_+]$ is a single point and
$\Gamma_{\va}\supseteq a_j\sZ$, we also have $\Delta<2|a_j|$ for every $j$ with
$a_j\neq0$.

These inequalities hold on
$\gS:=\{(\xi,\va) \mid \va\neq\vzero,\ \xi\in\Gamma_{\va}\}$. In order to extend them to $\sR \times \sR^{n}$, notice that every $(\xi,\va)$
is the midpoint of $(2\xi,\va_+)$ and $(0,\va_-)$, where
$\va_\pm:=\va\pm(2\xi-a_1)\ve_1$, and both points lie in $\gS$ whenever
$\va_+,\va_-\neq\vzero$, because $\va_+^\T\ve_1=2\xi$ by construction and
$0\in\Gamma_{\va_-}$. By concavity of $\Delta$ and affineness of $\rho_i$,
$\Delta\ge0$ and $\rho_i\ge0$ (for $\ell_i=0$) hold at all such points. The
excluded points, where $\va_+=\vzero$ or $\va_-=\vzero$, have
$\va=\alpha\ve_1$ and $\xi\in\{0,\alpha\}$ for some
$\alpha\in\sR$, so they lie
on two lines, and the inequalities hold everywhere by continuity. A concave
function bounded below on $\sR^{1+n}$ is constant, and so is an affine one.
Evaluating $\Delta$ at $(0,\varepsilon\ve_1)\in\gS$ gives $\Delta<2\varepsilon$
for every $\varepsilon>0$, so $\Delta\equiv0$, and evaluating $\rho_i$ at
$(0,\vzero)$ gives $\rho_i\equiv s_i$ for $\ell_i=0$.

So $H:=H_+=H_-$ is concave and convex, hence affine. At $(0,\vzero)$, where
$\rho_i=s_i$, it equals $\min_{\ell_i>0}s_i/\ell_i\ge0$ and
$\max_{\ell_i<0}s_i/\ell_i\le0$. Hence $H(0,\vzero)=0$, the row attaining the
minimum has $\ell_i>0$ and $s_i=0$, and $H$ is linear, say
$H(\xi,\va)=\kappa(\vt)\xi-\vq(\vt)^\T\va$. For $\ell_i>0$ the affine
function $\rho_i/\ell_i-H$ is nonnegative, and for $\ell_i<0$ it is
nonpositive. In both cases it is constant, equal to its value $s_i/\ell_i$
at $(0,\vzero)$. So $\rho_i-s_i=\ell_iH$ for every row, including those with
$\ell_i=0$, and comparing the coefficients of $\xi$ and of $\va$ gives
$\vg_i^\T\vt=\kappa(\vt)\ell_i$ and $h_i(\vt\va^\T)=-\ell_i\,\vq(\vt)^\T\va$,
which is the claim row by row.
\end{proof}

\begin{lemma}\label{lem:structure}
Every $p\in[m]$ has a copy row, that is, a row with
$\vlambda_i^\T=\theta_i\ve_p^\T$ and $\theta_i>0$, with $s_i=0$. There are
$\kappa\ge0$ and $\vq\in\sR^n$ with $\mG=\kappa\mLambda$ and
$\hmap(\mA)=-\mLambda\mA\vq$ for all $\mA$.
\end{lemma}

\begin{proof}
Fix $p$ and let $\vt=\tau\ve_p-\sum_{r\neq p}\ve_r\in\gT$ with $\tau>0$, so
that $\ell_i=\tau\Lambda_{i,p}-\sum_{r\neq p}\Lambda_{i,r}$. As
$\mLambda\ge0$, if $\tau$ is smaller than every positive number among the
ratios $\sum_{r\neq p}\Lambda_{i,r}/\Lambda_{i,p}$ over the rows with
$\Lambda_{i,p}>0$, then the rows with $\ell_i>0$ are exactly the copy rows
of $p$. By \cref{lem:hyperplane} such a row exists
with $s_i=0$.

Let $i$ and $i'$ be copy rows of $p\neq p'$. By \cref{lem:hyperplane},
$\kappa(\vt)=\vg_i^\T\vt/(\theta_it_p)=\vg_{i'}^\T\vt/(\theta_{i'}t_{p'})$
for $\vt\in\gT$, so
$\theta_{i'}t_{p'}\,\vg_i^\T\vt=\theta_it_p\,\vg_{i'}^\T\vt$ on the open set
$\gT$, hence as polynomials. As $t_p$ divides the right-hand side and not
$t_{p'}$, it divides the linear form $\vg_i^\T\vt$, so $\vg_i^\T\vt=\mu t_p$
for some $\mu\in\sR$, and $\kappa(\vt)=\mu/\theta_i$ is a constant $\kappa$. The same argument
applied to $-\vq(\vt)^\T\va=h_i(\vt\va^\T)/(\theta_it_p)$, for fixed $\va$,
shows that $\vq(\vt)=:\vq$ is constant. Now \cref{lem:hyperplane} gives
$\vg_i^\T\vt=\kappa\,\vlambda_i^\T\vt$ and
$h_i(\vt\va^\T)=-\vlambda_i^\T(\vt\va^\T)\vq$ for all rows, all $\vt$ in the
open set $\gT$ and all $\va$. Both sides are linear in $\vt$, so the
identities hold for all $\vt\in\sR^m$. The first gives $\mG=\kappa\mLambda$,
and the second gives $\hmap(\mA)=-\mLambda\mA\vq$, since the matrices
$\vt\va^\T$ span $\sR^{m\times n}$. Finally a copy row has
$\vg_i=\kappa\theta_i\ve_p\ge0$, so $\kappa\ge0$.
\end{proof}

By \cref{lem:structure}, \cref{eq:form2} becomes
\begin{equation}\label{eq:form3}
  \Phi(\mA,\vb,\vc)=\bigl(\mLambda\mA,\ \mLambda(\kappa\vb-\mA\vq)+\vs,\ \vc\bigr),
  \qquad \kappa\ge0 .
\end{equation}

\begin{lemma}\label{lem:G}
$\kappa=1$.
\end{lemma}

\begin{proof}
Fix $\vt\in\gT$ and $\vb\in\sR^m$ with $U>L$, where $U:=\min_{t_p>0}b_p/t_p$
and $L:=\max_{t_p<0}b_p/t_p$, for instance $b_p=\max(t_p,0)$. For
$N\in\mathbb N$ the instance $(\vt\,\ve_1^\T,\ N\vb)$ has feasible set
$\{\vx\in\sV \mid NL\le x_1\le NU\}$. By \cref{eq:form3}, row $i$ of its image
reads $\ell_iy_1\le N\kappa\,\vlambda_i^\T\vb-\ell_iq_1+s_i$, so the image
has feasible set $\{\vy\in\sV \mid L'_N\le y_1\le U'_N\}$ with
\[
  U'_N=\min_{\ell_i>0}\frac{N\kappa\,\vlambda_i^\T\vb+s_i}{\ell_i}-q_1,
  \qquad
  L'_N=\max_{\ell_i<0}\frac{N\kappa\,\vlambda_i^\T\vb+s_i}{\ell_i}-q_1,
\]
provided the rows with $\ell_i=0$, which read
$0\le N\kappa\,\vlambda_i^\T\vb+s_i$, hold.

For every row $i$,
$\vlambda_i^\T\vb-\ell_iU=\sum_p\Lambda_{i,p}t_p(b_p/t_p-U)$ is a sum of
nonnegative terms, because $b_p/t_p\ge U$ when $t_p>0$ and
$b_p/t_p\le L<U$ when $t_p<0$, and it vanishes for a copy row of a
constraint $p$ with $b_p/t_p=U$. Likewise $\vlambda_i^\T\vb-\ell_iL\ge0$,
with equality for a copy row of a constraint $p$ with $b_p/t_p=L$. In
particular, rows with $\ell_i=0$ have $\vlambda_i^\T\vb\ge0$, so their
constraints hold. Let
$C:=\max_{\ell_i\neq0}|s_i/\ell_i|$. For $\ell_i>0$,
\[
  \frac{N\kappa\,\vlambda_i^\T\vb+s_i}{\ell_i}
  =N\kappa U+N\kappa\,\frac{\vlambda_i^\T\vb-\ell_iU}{\ell_i}+\frac{s_i}{\ell_i}
  \ \ge\ N\kappa U,
\]
with the middle term vanishing for the copy row of a constraint $p$ with
$b_p/t_p=U$, so $N\kappa U-q_1\le U'_N\le N\kappa U-q_1+C$. Similarly
$N\kappa L-q_1-C\le L'_N\le N\kappa L-q_1$, and together
$|(U'_N-L'_N)-N\kappa(U-L)|\le2C$.

Now take $N$ with $N(U-L)\ge2$, so that the feasible set is nonempty, and
let $\Gamma:=\ve_1^\T\sV\in\{\sR,\sZ\}$. A witness $(\varphi,\gamma)$ of
$(\vt\ve_1^\T,N\vb,\ve_1)\sim\Phi(\vt\ve_1^\T,N\vb,\ve_1)$ satisfies
$\ve_1^\T\varphi(\vx)=x_1+\gamma$ on $F$ and maps $F$ onto $F'$, so the set
$\ve_1^\T F'=\Gamma\cap[L'_N,U'_N]$ is the translate of
$\ve_1^\T F=\Gamma\cap[NL,NU]$ by $\gamma$, and their diameters agree. The diameter of a nonempty set
$\Gamma\cap[L,U]$ lies in $(U-L-2,\,U-L]$, so
$|(U'_N-L'_N)-N(U-L)|<2$. With the bound above,
$|N\kappa(U-L)-N(U-L)|<2+2C$, and dividing by $N$ and letting $N\to\infty$
gives $\kappa=1$.
\end{proof}

By \cref{lem:G}, \cref{eq:form3} reads
$\Phi(\mA,\vb,\vc)=\bigl(\mLambda\mA,\ \mLambda(\vb-\mA\vq)+\vs,\ \vc\bigr)
=\Red(\mLambda,\mLambda\vb+\vs)\circ\Sub(\mI,\vq)\cdot(\mA,\vb,\vc)$, since
$\Sub(\mI,\vq)\cdot(\mA,\vb,\vc)=(\mA,\vb-\mA\vq,\vc)$. Here $\mLambda\ge0$,
$\vs\ge0$, and by \cref{lem:structure} every $p$ has a copy row $i$ with
$s_i=0$, that is, with $g_i(\vb)=\theta_ib_p$. So $(\mLambda,\vs)$ satisfies
the conditions of \cref{thm:redesc}, and by the sufficiency half of
\cref{app:redesc} the re-description preserves feasible sets and objectives.
Hence $\Sub(\mI,\vq)\cdot M$ has the same feasible set and objective as
$\Phi(M)$ for every $M$, so a witness of $M\sim\Phi(M)$ is one of
$M\sim\Sub(\mI,\vq)\cdot M$. Thus $\Sub(\mI,\vq)$ is valid and $\vq\in\sV$
by \cref{thm:sub}. Undoing the normalization,
$\Phi=\Red(\mLambda,\mLambda\vb+\vs)\circ\Sub(\mI,\vq)\circ\Sub(\mT,\vzero)
=\Red(\mLambda,\mLambda\vb+\vs)\circ\Sub(\mT,\mT\vq)$, as
$\Sub(\mI,\vq)\circ\Sub(\mT,\vzero)$ maps $(\mA,\vb,\vc)$ to
$(\mA\mT,\vb-\mA\mT\vq,\mT^\T\vc)$, and $\mT\vq\in\sV$ because $\mT$ has the
block form. \qed

In particular, if a constraint rewriting $\Red(\mLambda,\mLambda\vb+\vs)$
with affine $\vg$ is valid, then \cref{thm:complete} writes it as
$\Red(\mLambda',\mLambda'\vb+\vs')\circ\Sub(\mT,\vq)$ with $(\mLambda',\vs')$
as in \cref{thm:redesc}. Comparing the two rules on all instances gives
$\mT=\mI$, $\mLambda'=\mLambda$, $\vs'=\vs$ and $\vq=\vzero$, so the
rewriting is a re-description. This is the converse claimed after
\cref{eq:redesc}.

\section{Contribution of the substitutions}
\label{sec:ablation}

The two views of ReMILP differ by independent re-descriptions and by a substitution applied to the second one. To isolate the impact of the substitutions, as well as the equivariance machinery that we introduced to make use of them, we pre-train the same encoder without them in a purely contrastive way, so that the views differ by redundant constraints only. Each variable's embedding must then match its own embedding in the other view, and the negatives are embeddings of other variables. Node-level heads are trained for 20\,000 steps and the integrality gap head for 100\,000 steps. The results are reported in \Cref{fig:ablation-ft,fig:ablation-gap}.

Without substitutions the pre-trained encoder still converges faster than the supervised model, but it converges to the same value on every group, contrary to the original ReMILP which converges to $0.854 \pm 0.020$ of the supervised model's KL on solution prediction and $0.918 \pm 0.050$ on constraint activity. Redundant constraints alone thus account for the speed-up of \cref{sec:finetuned}, but the lasting advantage on mixed-integer families comes from the substitutions.

The integrality gap shows the same split between speed and final error. Without substitutions the encoder descends as fast as ReMILP and far ahead of the supervised model, but it settles at $0.042 \pm 0.002$, above ReMILP's $0.038 \pm 0.001$ and within one standard deviation of the supervised model's $0.045 \pm 0.004$.

\begin{figure}[t]
  \centering
  \begin{subfigure}[t]{0.24\linewidth}
    \centering
    \includegraphics[width=\linewidth]{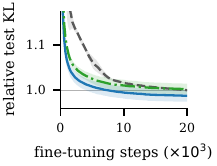}
    \caption{Solution, pure ILP}
    \label{fig:abl-solution-ilp}
  \end{subfigure}\hfill
  \begin{subfigure}[t]{0.24\linewidth}
    \centering
    \includegraphics[width=\linewidth]{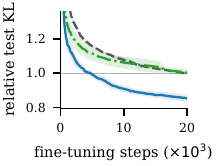}
    \caption{Solution, MILP}
    \label{fig:abl-solution-milp}
  \end{subfigure}\hfill
  \begin{subfigure}[t]{0.24\linewidth}
    \centering
    \includegraphics[width=\linewidth]{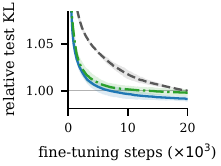}
    \caption{Activity, pure ILP}
    \label{fig:abl-activity-ilp}
  \end{subfigure}\hfill
  \begin{subfigure}[t]{0.24\linewidth}
    \centering
    \includegraphics[width=\linewidth]{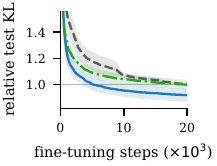}
    \caption{Activity, MILP}
    \label{fig:abl-activity-milp}
  \end{subfigure}
  \par\vspace{0.2em}
  \includegraphics[width=0.62\linewidth]{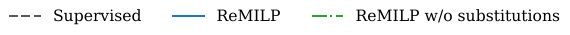}
  \caption{Fine-tuning ReMILP with and without substitutions on the node-level tasks, 20\,000 steps. Each panel shows the geometric mean over the pairs of a group of the test KL at the checkpoint selected on validation, divided by the value the supervised model of the same seed reaches at the end of its budget. Shaded bands are   one standard deviation over three seeds, and the horizontal line marks parity with the best supervised model.}
  \label{fig:ablation-ft}
\end{figure}

\begin{figure}[t]
  \centering
  \includegraphics[width=0.5\linewidth]{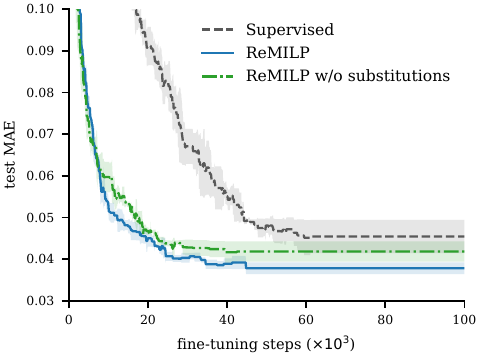}
  \caption{Fine-tuning ReMILP with and without substitutions on the integrality gap: test MAE, mean and one standard deviation over three
  seeds, 100\,000-step budget on the full training set.}
  \label{fig:ablation-gap}
\end{figure}

\section{Datasets}
\label{app:datasets}
In this work, we make use of two datasets for training and evaluating our models:
\subsection{MILP-Evolve}
We made use of the integrality gap prediction dataset of MILP-Evolve \citep{Li2025-iclr} for both pre-training and integrality gap prediction. MILP instance families were filtered by feasibility, objective boundedness, and primal gap convergence. Instance families with SOS or indicator constraints were eliminated. Instance families that allowed instances with zero variables or zero constraints were also filtered out. After the filtering step, only 42k instances remained, which were randomly split into a pre-training set (70\%) and an evaluation set (30\%). The pre-training set was further split into a 90\% training subset, and a 10\% subset for validation. Similarly, the evaluation set was split into training (70\%), validation (10\%) and test (20\%) subsets, which were all used for integrality gap prediction. We additionally recomputed the integrality gap targets using the well-conditioned formula
$$
    \frac{|\mathrm{LP}-\mathrm{IP}|}{ \max(|\mathrm{LP}|,|\mathrm{IP}|) + 10^{-5}}
$$
rather than $\frac{\left| \mathrm{LP}-\mathrm{IP}\right|}{ |\mathrm{LP}|}$ which was used by MILP-Evolve as the latter led to exploding gradients in our setup.

\subsection{Distributional MIPLIB}
\paragraph{Instance family selection} We use all instance families with a difficulty of Easy, Medium or Hard and containing at least 300 instances from the Distributional MIPLIB dataset \citep{Huang2024-arxiv}. This leaves us with 23 instance families across 10 problem classes, summarized in \cref{tab:dmiplib}.

\begin{table}[htbp]
    \centering
    \caption{Instance families used from Distributional MIPLIB.}
    \label{tab:dmiplib}
    \begin{tabular}{llccc}
        \toprule
        Problem class & Abbreviation & Easy & Medium & Hard \\
        \midrule
        Combinatorial Auctions & CA & \checkmark & \checkmark & \\
        Maximum Independent Set & IS & \checkmark & \checkmark & \\
        Set Covering & SC & \checkmark & \checkmark & \checkmark \\
        Minimum Vertex Cover & VC & \checkmark & \checkmark & \checkmark \\
        Generalized Independent Set Problem & GISP & \checkmark & \checkmark & \checkmark \\
        Capacitated Facility Location Problem & CFLP & \checkmark & \checkmark & \\
        Load Balancing & LB & & & \checkmark \\
        Neural Network Verification & NNV & \checkmark & & \\
        Optimal Transmission Switching & OTS & \checkmark & \checkmark & \checkmark \\
        Middle-Mile Consolidation Network & MMCN & & BI, BC & BI \\
        \bottomrule
    \end{tabular}
\end{table}

For each family, we separate our instances into 200 for training, 50 for validation and 50 for testing, which are solved in the manner described below. CFLP, LB, NNV, OTS and MMCN medium-BC contain both integer and continuous variables. The other families are pure ILPs.
\paragraph{Node-level targets}
In order to compute the binary solution and constraint activity target values, we solve each family's 300 instances using the Gurobi solver \citep{gurobi}. For each instance, we use a 1-hour time limit and retain the best 300 solutions during that time. Let $\vx^{(1)},\ldots,\vx^{(S)}$ be the solution pool of an instance and $z_s = \vc^\top \vx^{(s)}$ their objective values. Following \citet{Nair2020-arxiv}, we weight the pool by objective quality:
$$ w_s \propto \exp(-\, z_s), $$
and define both our variable-level and constraint-level targets as marginals of this distribution:
$$\hat y_j = \mathbb{E}_{w}\big[x_j\big]$$
for binary solution prediction (the weighted frequency with which a binary variable $j$ takes value $1$), and
$$\hat d_i = \mathbb{P}_{w}\big[\,|\va_i^\top \vx - b_i| \le \varepsilon \,\big]$$
for constraint activity prediction (the weighted frequency with which constraint $i$ is binding, up to a small tolerance $\varepsilon = 10^{-6}$).

\section{Implementation details}
\label{app:details}
\subsection{Preprocessing}
Before encoding, we normalize each constraint row and its bound by the same positive factor, such that every single coefficient of $\mA$ and component of $\vb$ is bounded by $1$. This normalization is principled insofar as it preserves a MILP's feasible region. Objective coefficients are scaled in the same way, which preserves optimal solutions and integrality gap. Our implementation of the various augmentations described in this paper ensures that instances stay row-normalized, which prevents distribution shift.
\subsection{Training}
\label{app:training}
Across pre-training and downstream tasks, all optimization is done with Adam \citep{kingma2017adammethodstochasticoptimization}. We use dynamic batches with a maximum size of 200\,000 nodes and 1\,200\,000 edges. Any instance exceeding either of the two limits is excluded from the training. All pre-training is done for 20\,000 steps at a learning rate of $10^{-3}$. For all downstream tasks, we use a fixed learning rate of $10^{-4}$. Binary solution and constraint activity prediction are run for 50\,000 gradient steps, and integrality gap prediction is run for 100\,000 gradient steps. All three tasks have early stopping if the best validation loss does not improve for 5\,000 consecutive steps.

\subsection{Instance pooling}
\label{app:pooling}
The integrality gap is predicted from a single embedding of the instance, obtained by pooling the node embeddings. Variables and constraints are pooled separately, with a coordinate-wise attention: every output coordinate has its own softmax over the nodes. For the variables, a two-layer MLP $s$ maps each embedding $\vh_i(M)\in\sR^d$ to a score vector $s(\vh_i(M))\in\sR^{p}$, and a linear map $\mW$ maps it to a value vector $\mW\vh_i(M)\in\sR^{p}$. The pooled variable vector is
$$
\vz_V(M)=\sum_{i=1}^{n}\va_i\odot\mW\vh_i(M),\qquad
a_{i,k}=\frac{\exp\bigl(s_k(\vh_i(M))\bigr)}{\sum_{i'=1}^{n}\exp\bigl(s_k(\vh_{i'}(M))\bigr)},
$$
so each coordinate is a convex combination of the corresponding value coordinates of the variables, weighted by that coordinate's own attention distribution. The constraint embeddings $\vg_j(M)$ are pooled in the same way by separate parameters into $\vz_C(M)$. The instance embedding is the output of a two-layer MLP applied to $[\vz_V(M);\vz_C(M)]$. We use $d=128$, $p=256$ for each node type, and an instance embedding of dimension 256. The pooling and the MLP are trained together with the gap head, including when the encoder is frozen. FORGE\textsuperscript{\textdagger} instead represents an instance by the histogram of its nodes' codebook assignments, and the variant marked "+ attn." replaces that histogram with the pooling described here.

\subsection{Hypernetwork}
\label{app:hypernetwork}

The hypernetwork $\varrho_\phi$ of \cref{eq:local_pred} is a single affine layer. For a
transformed variable $i\in\gI$, it reads the descriptor
\[
  t_i=\bigl(\log|u_i|,\ \1[u_i<0],\ \lambda_i,\ \mu_i,\ \sg(\vh_{j(i)}(M))\bigr)\in\sR^{4+d},
\]
in which the scaling enters through its log-magnitude and its sign so that complements and
scalings of either sign are represented on a common scale. The partner embedding is
replaced by $\bm0$ when $x_i$ has no partner, that is, for binary variables and for
variables that are only scaled. The layer outputs the $d^2$ entries of the operator,
\[
  \varrho_\phi(t_i)=\operatorname{mat}\bigl(\mW_\phi t_i+\vb_\phi\bigr)\in\sR^{d\times d},
  \qquad \mW_\phi\in\sR^{d^2\times(4+d)},\ \vb_\phi\in\sR^{d^2},
\]
where $\operatorname{mat}$ reshapes a vector of length $d^2$ into a $d\times d$ matrix.
The weights $\mW_\phi$ are initialised from $\gN(0,0.02^2)$ and the bias to
$\vb_\phi=\operatorname{vec}(\mI_d)$, so that every operator starts close to the identity
and training begins from the invariant prediction $\vh_i(M')\approx\vh_i(M)$. With $d=128$ the hypernetwork has about $2.2$ million parameters. Among the action negatives ($1\le k\le K$ in \cref{eq:loss}), those drawn from a variable with the same descriptor parameters and the same partner as $i$ would produce the positive itself, and they are removed from the softmax. The hypernetwork is used only during pre-training, and downstream heads read the encoder's embeddings directly.

\subsection{Sampling augmentations}
\label{app:aug}

Both views of \cref{eq:reformulations} are sampled afresh for every instance at every
step. \Cref{alg:sub} draws the substitution $\Sub(\mT',\vq')$ of the second view, and
\cref{alg:redesc} draws each of the two re-descriptions. The fraction $|\gI|/n$, the bounds
$b_u$ and $b_\lambda$, the fraction $m'/m-1$ and the number $r$ of constraints combined per
redundant row are the hyperparameters of \cref{tab:hparams}.

Every transformed variable $i\in\gI$ receives the local substitution
\cref{eq:affine_substitution}. A binary variable must remain binary, and the only
substitution with this property other than the identity is the complement
$x_i\leftarrow1-x_i$, that is $u_i=-1$, $\lambda_i=0$ and $\mu_i=1$, which involves no
partner. A transformed binary variable is therefore complemented with probability
$\frac12$ and otherwise left unchanged, and it stays in $\gI$ in both cases, with the
identity as its substitution in the second case. Every other transformed variable takes a
partner $j(i)$ drawn from $\gJ$, so no partner is itself transformed and $\mT'$ has the
structure required by \cref{thm:sub}. An integer variable takes an integer partner, which
may be binary, while a continuous variable may take a partner of any type. Binary
variables in $\gJ$ are therefore eligible partners of both integer and continuous
variables. We write $\sZ_b\coloneq\{z\in\sZ:1\le|z|\le\lfloor b\rfloor\}$ for the
nonzero integers of absolute value at most $b$.

\begin{algorithm}[H]
  \caption{Sampling the substitution $\Sub(\mT',\vq')$}
  \label{alg:sub}
  \begin{algorithmic}[1]
    \Require instance $M_0$ over $\sV$, fraction $\kappa=|\gI|/n$, bounds $b_u,b_\lambda$
    \State $\gC\gets$ variables with a nonzero coefficient in some constraint of $M_0$
    \State draw $\gI\subseteq\gC$ uniformly without replacement, with
      $|\gI|=\operatorname{round}(\kappa|\gC|)$
    \State $\gJ\gets[n]\setminus\gI$
    \For{$i\in\gI$}
      \If{$x_i$ is binary}
        \State $(u_i,\lambda_i,\mu_i)\gets(-1,0,1)$ with probability $\frac12$,
          else $(1,0,0)$
      \ElsIf{$x_i$ is integer}
        \State draw $j(i)$ uniformly among the integer variables of $\gJ$
        \State draw $u_i$ uniformly from $\{-1,1\}$, and $\lambda_i,\mu_i$ uniformly from
          $\sZ_{b_\lambda}$
      \Else
        \State draw $j(i)$ uniformly from $\gJ$
        \State draw $\log|u_i|$ uniformly from $[-\log b_u,\log b_u]$ and the sign of $u_i$
          uniformly
        \State draw $\lambda_i,\mu_i$ uniformly from $[-b_\lambda,b_\lambda]$
      \EndIf
    \EndFor
    \State \Return $\mT',\vq'$ assembled from $(u_i,\lambda_i,\mu_i,j(i))_{i\in\gI}$ as
      in \cref{eq:affine_substitution}
  \end{algorithmic}
\end{algorithm}

A variable left without an eligible partner, which only happens when $\gJ$ holds no
variable of the required type, is scaled but not transvected ($\lambda_i=\mu_i=0$).

Each re-description keeps the $m$ existing constraints and appends $m'-m$ redundant ones.
A redundant constraint is a convex combination of $r$ existing constraints, drawn with
replacement, whose right-hand side is then relaxed (\cref{alg:redesc}). Every constraint
of $M_0$ is scaled by a positive factor so that the largest absolute value among its
coefficients and its right-hand side is $1$, and the constraints of
$\Sub(\mT',\vq')\cdot M_0$ are rescaled in the same way before its re-description is
drawn. A positive scaling of the constraints is itself a re-description, and it gives
$\vlambda_i^\T\vb\le1$ for every convex combination $\vlambda_i$.

\begin{algorithm}[H]
  \caption{Sampling the re-description $\sRed(\mLambda,\mLambda\vb+\vs)$}
  \label{alg:redesc}
  \begin{algorithmic}[1]
    \Require instance $M=(\mA,\vb,\vc)\in\Inst_m$ with scaled constraints, fraction
      $\rho=m'/m-1$, number of combined constraints $r$
    \State $m'\gets m+\operatorname{round}(\rho m)$
    \For{$i=m+1,\ldots,m'$}
      \State draw $p_1,\ldots,p_r$ uniformly from $[m]$ with replacement
      \State draw $(w_1,\ldots,w_r)$ uniformly from the probability simplex
      \State $\vlambda_i\gets\sum_{t=1}^{r}w_t\ve_{p_t}$
      \State draw $\nu_i$ uniformly from $[0,1]$ and set $s_i\gets\nu_i\,(1-\vlambda_i^\T\vb)$
    \EndFor
    \State $\mLambda\gets$ the rows of $\mI_m$ followed by
      $\vlambda_{m+1}^\T,\ldots,\vlambda_{m'}^\T$, and
      $\vs\gets(0,\ldots,0,s_{m+1},\ldots,s_{m'})$
    \State \Return $\sRed(\mLambda,\mLambda\vb+\vs)$
  \end{algorithmic}
\end{algorithm}

The relaxed right-hand side $\vlambda_i^\T\vb+s_i$ thus lies uniformly between the
implied value $\vlambda_i^\T\vb$ and $1$. The identity block of $\mLambda$ keeps a copy of every constraint of $M$, and
$\mLambda\ge0$ and $\vs\ge0$, so the re-description satisfies the conditions of
\cref{thm:redesc}.

\subsection{Hyperparameter tuning}
\label{app:hparams}

We tuned the pre-training hyperparameters of ReMILP and FORGE\textsuperscript{\textdagger}
separately with the multivariate TPE sampler of Optuna \citep{optuna_2019}, starting with
15 random warmup trials, for a total of 56 trials per model. The learning rates and numbers of steps were fixed. Each trial pre-trains the
encoder with three seeds and scores it on 14 cells, which are binary solution and constraint activity prediction on seven instance families: CA medium, SC hard, CFLP medium, OTS easy, MMCN medium-BC, NNV easy, and LB hard. On each cell the encoder is frozen and a head is trained for 5\,000 steps on training instances, and the cell is scored by its best validation KL.

Let $r_{s,c}$ be that validation KL for seed $s$ and cell $c$, divided by the value the random encoder reaches on the same cell and seed. Denote by $G_s$ the geometric mean over all cells $c$ for a fixed seed $s$. A trial minimizes
\[
  \operatorname{mean}_s G_s+\frac{1}{2}\operatorname{std}_s G_s
  +\frac{2}{|S|\,|C|}\sum_{s,c}\max(r_{s,c}-1,0)
\]
The second term penalizes variation across seeds and the third penalizes cells on which the encoder does worse than the random encoder. \Cref{tab:hparams} lists the search space and the selected
configurations.

\begin{table}[htbp]
  \centering
  \small
  \caption{\textbf{Pre-training search space and selected values.} Ranges marked log are sampled on a logarithmic scale.}
  \label{tab:hparams}
  \begin{tabular}{@{}llll@{}}
    \toprule
    Model & Hyperparameter & Search space & Selected \\
    \midrule
    \multirow{7}{*}{ReMILP}
      & Temperature $\tau$ & $[0.05, 0.5]$, log & $0.090$ \\
      & Negatives of each type $K$ & $\{7, 15, 31, 63\}$ & $63$ \\
      & Fraction of transformed variables $\kappa = |\gI|/n$ & $\{0.01, 0.02, 0.05\}$ & $0.05$ \\
      & Fraction of redundant constraints $\rho = \frac{m'}{m} - 1$ & $\{0, 0.05, 0.1, 0.2\}$ & $0.2$ \\
      & Constraints combined per redundant row & $\{2, 3\}$ & $2$ \\
      & Bound $b_\lambda$ on $\lambda_i,\mu_i$ & $[1.5, 6]$, log & $4.01$ \\
      & Bound $b_u$ on $u_i$ & $[1, 4]$, log & $1.12$ \\
    \midrule
    \multirow{3}{*}{FORGE\textsuperscript{\textdagger}}
      & Codebook size & $\{64, 128, 256, 1024\}$ & $256$ \\
      & Codebook EMA decay & $\{0.8, 0.9, 0.95, 0.99\}$ & $0.99$ \\
      & Commitment weight & $[0.01, 2]$, log & $0.046$ \\
    \bottomrule
  \end{tabular}
\end{table}

\clearpage

\section{Per-family results}

\label{app:families}
\Cref{tab:app-solution,tab:app-activity} give the frozen measurements behind \cref{tab:frozen} without aggregation, as the test KL of each instance family rather than a ratio to the random encoder, and
\cref{fig:app-solution-1,fig:app-solution-2,fig:app-activity-1,fig:app-activity-2} do the same for the fine-tuned runs of \cref{fig:ft-nodes}. The ratios of the main text are geometric means of these normalized values over instance families.

\begin{table}[t]
  \centering
  \small
  \setlength{\tabcolsep}{6pt}
  \caption{\textbf{Binary solution prediction.} Frozen encoders, test KL per instance family, mean $\pm$ one standard deviation over three seeds. A value is bold when it is not within one standard deviation of the runner up.}
  \label{tab:app-solution}
  \begin{tabular}{@{}ll ccc@{}}
    \toprule
    Class & Difficulty & Random & ReMILP & FORGE\textsuperscript{\textdagger} \\
    \midrule
    IS & easy & $0.588 \pm 0.000$ & $0.588 \pm 0.000$ & $0.588 \pm 0.000$ \\
     & medium & $0.599 \pm 0.000$ & $0.599 \pm 0.000$ & $0.599 \pm 0.000$ \\
    CA & easy & $0.242 \pm 0.001$ & $0.212 \pm 0.001$ & $0.217 \pm 0.005$ \\
     & medium & $0.243 \pm 0.001$ & $0.210 \pm 0.001$ & $0.216 \pm 0.005$ \\
    SC & easy & $0.057 \pm 0.000$ & $0.044 \pm 0.002$ & $0.047 \pm 0.002$ \\
     & medium & $0.060 \pm 0.000$ & $0.047 \pm 0.002$ & $0.052 \pm 0.003$ \\
     & hard & $0.065 \pm 0.000$ & $\mathbf{0.051 \pm 0.002}$ & $0.058 \pm 0.004$ \\
    VC & easy & $0.574 \pm 0.006$ & $0.505 \pm 0.008$ & $0.484 \pm 0.017$ \\
     & medium & $0.582 \pm 0.006$ & $0.512 \pm 0.010$ & $0.496 \pm 0.016$ \\
     & hard & $0.370 \pm 0.000$ & $0.371 \pm 0.000$ & $0.370 \pm 0.000$ \\
    GISP & easy & $0.172 \pm 0.001$ & $0.170 \pm 0.000$ & $0.170 \pm 0.000$ \\
     & medium & $0.136 \pm 0.000$ & $0.135 \pm 0.000$ & $0.135 \pm 0.000$ \\
     & hard & $0.117 \pm 0.000$ & $0.115 \pm 0.000$ & $0.116 \pm 0.000$ \\
    MMCN & medium-BI & $0.310 \pm 0.004$ & $0.252 \pm 0.013$ & $0.285 \pm 0.020$ \\
     & hard-BI & $0.336 \pm 0.013$ & $\mathbf{0.250 \pm 0.011}$ & $0.283 \pm 0.021$ \\
    \midrule
    CFLP & easy & $0.245 \pm 0.003$ & $0.209 \pm 0.023$ & $0.228 \pm 0.016$ \\
     & medium & $0.213 \pm 0.000$ & $0.178 \pm 0.043$ & $0.196 \pm 0.015$ \\
    OTS & easy & $0.293 \pm 0.018$ & $\mathbf{0.167 \pm 0.016}$ & $0.210 \pm 0.025$ \\
     & medium & $0.339 \pm 0.023$ & $\mathbf{0.188 \pm 0.011}$ & $0.242 \pm 0.036$ \\
     & hard & $0.268 \pm 0.002$ & $\mathbf{0.245 \pm 0.007}$ & $0.261 \pm 0.004$ \\
    MMCN & medium-BC & $0.071 \pm 0.004$ & $\mathbf{0.022 \pm 0.000}$ & $0.025 \pm 0.000$ \\
    NNV & easy & $0.585 \pm 0.007$ & $0.553 \pm 0.006$ & $0.555 \pm 0.000$ \\
    LB & hard & $0.512 \pm 0.039$ & $0.436 \pm 0.045$ & $0.385 \pm 0.055$ \\
    \bottomrule
  \end{tabular}
\end{table}

\begin{table}[t]
  \centering
  \small
  \setlength{\tabcolsep}{6pt}
  \caption{Constraint activity prediction. Frozen encoders, test KL per instance family,
  mean $\pm$ one standard deviation over three seeds. Heads trained for 50\,000 steps. The upper block holds the pairs with no continuous
  variables. A value is bold when not within one standard deviation of the runner up.}
  \label{tab:app-activity}
  \begin{tabular}{@{}ll ccc@{}}
    \toprule
    Class & Difficulty & Random & ReMILP & FORGE\textsuperscript{\textdagger} \\
    \midrule
    IS & easy & $0.583 \pm 0.000$ & $0.583 \pm 0.000$ & $0.583 \pm 0.000$ \\
     & medium & $0.595 \pm 0.000$ & $0.595 \pm 0.000$ & $0.595 \pm 0.000$ \\
    CA & easy & $0.384 \pm 0.004$ & $0.312 \pm 0.002$ & $0.318 \pm 0.008$ \\
     & medium & $0.394 \pm 0.003$ & $0.313 \pm 0.002$ & $0.321 \pm 0.008$ \\
    SC & easy & $0.389 \pm 0.012$ & $0.286 \pm 0.003$ & $0.287 \pm 0.003$ \\
     & medium & $0.322 \pm 0.008$ & $0.227 \pm 0.004$ & $0.229 \pm 0.005$ \\
     & hard & $0.257 \pm 0.002$ & $0.175 \pm 0.003$ & $0.176 \pm 0.005$ \\
    VC & easy & $0.631 \pm 0.001$ & $0.568 \pm 0.007$ & $0.571 \pm 0.035$ \\
     & medium & $0.639 \pm 0.001$ & $0.577 \pm 0.008$ & $0.580 \pm 0.035$ \\
     & hard & $0.520 \pm 0.000$ & $0.520 \pm 0.000$ & $0.520 \pm 0.000$ \\
    GISP & easy & $0.577 \pm 0.000$ & $0.576 \pm 0.000$ & $0.576 \pm 0.000$ \\
     & medium & $0.552 \pm 0.000$ & $0.553 \pm 0.000$ & $0.553 \pm 0.000$ \\
     & hard & $0.523 \pm 0.003$ & $0.521 \pm 0.000$ & $0.522 \pm 0.000$ \\
    MMCN & medium-BI & $0.183 \pm 0.002$ & $0.154 \pm 0.003$ & $0.161 \pm 0.006$ \\
     & hard-BI & $0.158 \pm 0.003$ & $\mathbf{0.128 \pm 0.002}$ & $0.136 \pm 0.003$ \\
    \midrule
    CFLP & easy & $0.235 \pm 0.003$ & $0.214 \pm 0.014$ & $0.221 \pm 0.005$ \\
     & medium & $0.208 \pm 0.001$ & $0.182 \pm 0.039$ & $0.189 \pm 0.023$ \\
    OTS & easy & $0.150 \pm 0.006$ & $0.134 \pm 0.004$ & $0.130 \pm 0.008$ \\
     & medium & $0.156 \pm 0.005$ & $0.135 \pm 0.006$ & $0.133 \pm 0.006$ \\
     & hard & $0.125 \pm 0.003$ & $\mathbf{0.109 \pm 0.002}$ & $0.122 \pm 0.004$ \\
    MMCN & medium-BC & $0.311 \pm 0.002$ & $0.219 \pm 0.007$ & $0.225 \pm 0.011$ \\
    NNV & easy & $0.049 \pm 0.001$ & $0.046 \pm 0.000$ & $0.046 \pm 0.000$ \\
    LB & hard & $0.554 \pm 0.003$ & $0.445 \pm 0.035$ & $0.410 \pm 0.020$ \\
    \bottomrule
  \end{tabular}
\end{table}

\begin{figure}[p]
  \centering
  \begin{subfigure}[t]{0.32\linewidth}
    \centering
    \includegraphics[width=\linewidth]{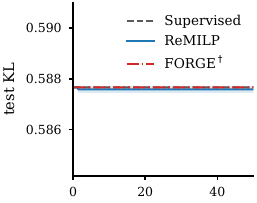}
    \caption{IS, easy}
  \end{subfigure}\hfill
  \begin{subfigure}[t]{0.32\linewidth}
    \centering
    \includegraphics[width=\linewidth]{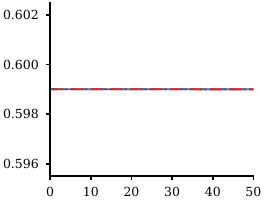}
    \caption{IS, medium}
  \end{subfigure}\hfill
  \begin{subfigure}[t]{0.32\linewidth}
    \centering
    \includegraphics[width=\linewidth]{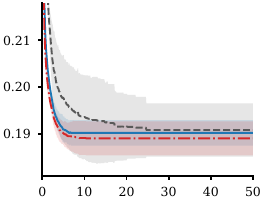}
    \caption{CA, easy}
  \end{subfigure}\hfill
  \par\vspace{0.4em}
  \begin{subfigure}[t]{0.32\linewidth}
    \centering
    \includegraphics[width=\linewidth]{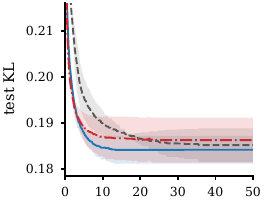}
    \caption{CA, medium}
  \end{subfigure}\hfill
  \begin{subfigure}[t]{0.32\linewidth}
    \centering
    \includegraphics[width=\linewidth]{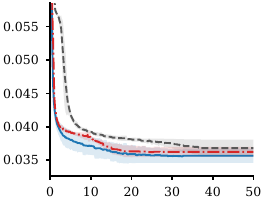}
    \caption{SC, easy}
  \end{subfigure}\hfill
  \begin{subfigure}[t]{0.32\linewidth}
    \centering
    \includegraphics[width=\linewidth]{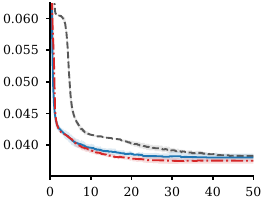}
    \caption{SC, medium}
  \end{subfigure}\hfill
  \par\vspace{0.4em}
  \begin{subfigure}[t]{0.32\linewidth}
    \centering
    \includegraphics[width=\linewidth]{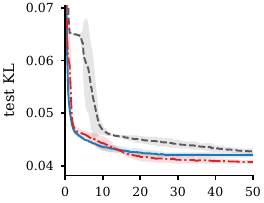}
    \caption{SC, hard}
  \end{subfigure}\hfill
  \begin{subfigure}[t]{0.32\linewidth}
    \centering
    \includegraphics[width=\linewidth]{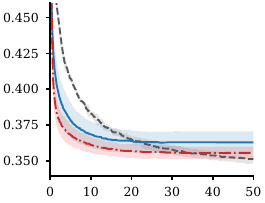}
    \caption{VC, easy}
  \end{subfigure}\hfill
  \begin{subfigure}[t]{0.32\linewidth}
    \centering
    \includegraphics[width=\linewidth]{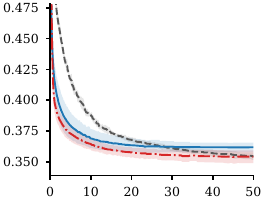}
    \caption{VC, medium}
  \end{subfigure}\hfill
  \par\vspace{0.4em}
  \begin{subfigure}[t]{0.32\linewidth}
    \centering
    \includegraphics[width=\linewidth]{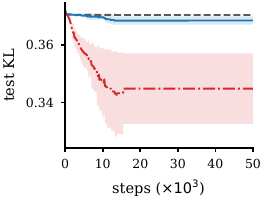}
    \caption{VC, hard}
  \end{subfigure}\hfill
  \begin{subfigure}[t]{0.32\linewidth}
    \centering
    \includegraphics[width=\linewidth]{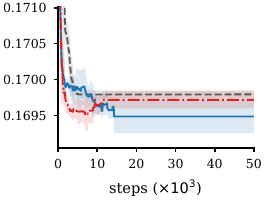}
    \caption{GISP, easy}
  \end{subfigure}\hfill
  \begin{subfigure}[t]{0.32\linewidth}
    \centering
    \includegraphics[width=\linewidth]{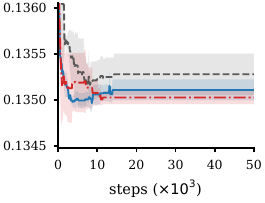}
    \caption{GISP, medium}
  \end{subfigure}\hfill
  \par\vspace{0.4em}
  \caption{\textbf{Fine-tuning on binary solution prediction.} Part 1 of 2. Test KL of the checkpoint selected on validation,
  one panel per instance family, mean and one standard deviation over three seeds. Each panel carries its own vertical scale.}
  \label{fig:app-solution-1}
\end{figure}

\begin{figure}[p]
  \centering
  \begin{subfigure}[t]{0.32\linewidth}
    \centering
    \includegraphics[width=\linewidth]{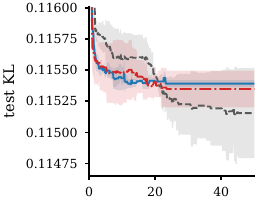}
    \caption{GISP, hard}
  \end{subfigure}\hfill
  \begin{subfigure}[t]{0.32\linewidth}
    \centering
    \includegraphics[width=\linewidth]{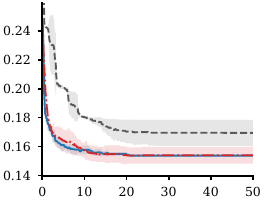}
    \caption{CFLP, easy}
  \end{subfigure}\hfill
  \begin{subfigure}[t]{0.32\linewidth}
    \centering
    \includegraphics[width=\linewidth]{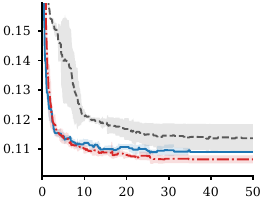}
    \caption{CFLP, medium}
  \end{subfigure}\hfill
  \par\vspace{0.4em}
  \begin{subfigure}[t]{0.32\linewidth}
    \centering
    \includegraphics[width=\linewidth]{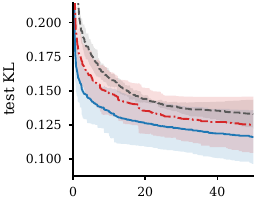}
    \caption{OTS, easy}
  \end{subfigure}\hfill
  \begin{subfigure}[t]{0.32\linewidth}
    \centering
    \includegraphics[width=\linewidth]{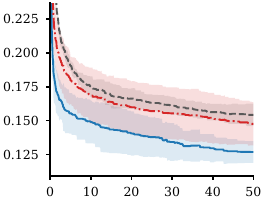}
    \caption{OTS, medium}
  \end{subfigure}\hfill
  \begin{subfigure}[t]{0.32\linewidth}
    \centering
    \includegraphics[width=\linewidth]{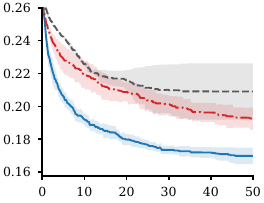}
    \caption{OTS, hard}
  \end{subfigure}\hfill
  \par\vspace{0.4em}
  \begin{subfigure}[t]{0.32\linewidth}
    \centering
    \includegraphics[width=\linewidth]{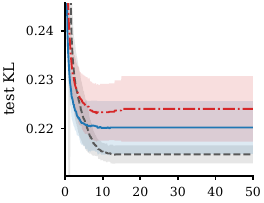}
    \caption{MMCN, medium-BI}
  \end{subfigure}\hfill
  \begin{subfigure}[t]{0.32\linewidth}
    \centering
    \includegraphics[width=\linewidth]{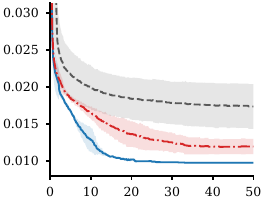}
    \caption{MMCN, medium-BC}
  \end{subfigure}\hfill
  \begin{subfigure}[t]{0.32\linewidth}
    \centering
    \includegraphics[width=\linewidth]{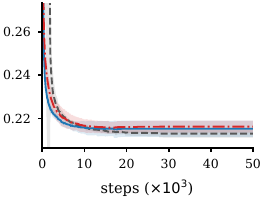}
    \caption{MMCN, hard-BI}
  \end{subfigure}\hfill
  \par\vspace{0.4em}
  \begin{subfigure}[t]{0.32\linewidth}
    \centering
    \includegraphics[width=\linewidth]{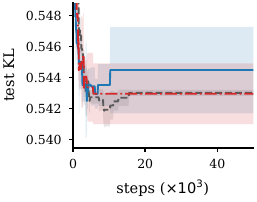}
    \caption{NNV, easy}
  \end{subfigure}\hfill
  \begin{subfigure}[t]{0.32\linewidth}
    \centering
    \includegraphics[width=\linewidth]{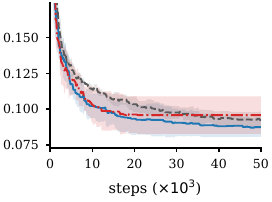}
    \caption{LB, hard}
  \end{subfigure}\hfill
  \caption{\textbf{Fine-tuning on binary solution prediction.} Part 2 of 2. Test KL of the checkpoint selected on validation,
  one panel per instance family, mean and one standard deviation over
  three seeds. Each panel carries its own vertical scale.}
  \label{fig:app-solution-2}
\end{figure}

\begin{figure}[p]
  \centering
  \begin{subfigure}[t]{0.32\linewidth}
    \centering
    \includegraphics[width=\linewidth]{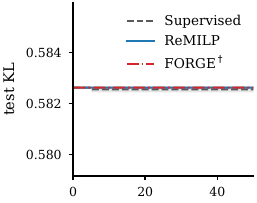}
    \caption{IS, easy}
  \end{subfigure}\hfill
  \begin{subfigure}[t]{0.32\linewidth}
    \centering
    \includegraphics[width=\linewidth]{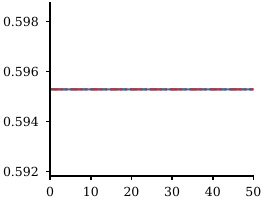}
    \caption{IS, medium}
  \end{subfigure}\hfill
  \begin{subfigure}[t]{0.32\linewidth}
    \centering
    \includegraphics[width=\linewidth]{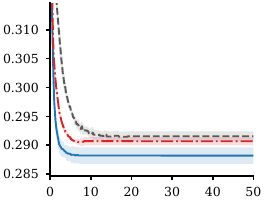}
    \caption{CA, easy}
  \end{subfigure}\hfill
  \par\vspace{0.4em}
  \begin{subfigure}[t]{0.32\linewidth}
    \centering
    \includegraphics[width=\linewidth]{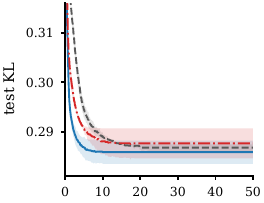}
    \caption{CA, medium}
  \end{subfigure}\hfill
  \begin{subfigure}[t]{0.32\linewidth}
    \centering
    \includegraphics[width=\linewidth]{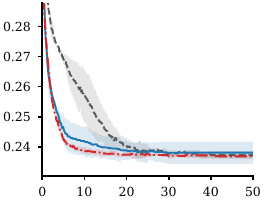}
    \caption{SC, easy}
  \end{subfigure}\hfill
  \begin{subfigure}[t]{0.32\linewidth}
    \centering
    \includegraphics[width=\linewidth]{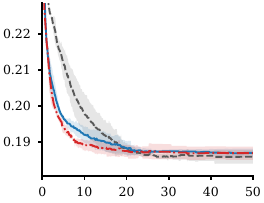}
    \caption{SC, medium}
  \end{subfigure}\hfill
  \par\vspace{0.4em}
  \begin{subfigure}[t]{0.32\linewidth}
    \centering
    \includegraphics[width=\linewidth]{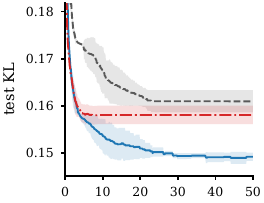}
    \caption{SC, hard}
  \end{subfigure}\hfill
  \begin{subfigure}[t]{0.32\linewidth}
    \centering
    \includegraphics[width=\linewidth]{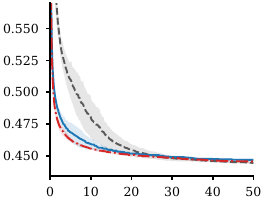}
    \caption{VC, easy}
  \end{subfigure}\hfill
  \begin{subfigure}[t]{0.32\linewidth}
    \centering
    \includegraphics[width=\linewidth]{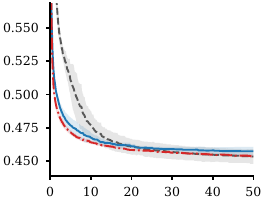}
    \caption{VC, medium}
  \end{subfigure}\hfill
  \par\vspace{0.4em}
  \begin{subfigure}[t]{0.32\linewidth}
    \centering
    \includegraphics[width=\linewidth]{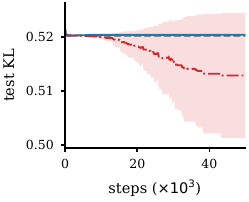}
    \caption{VC, hard}
  \end{subfigure}\hfill
  \begin{subfigure}[t]{0.32\linewidth}
    \centering
    \includegraphics[width=\linewidth]{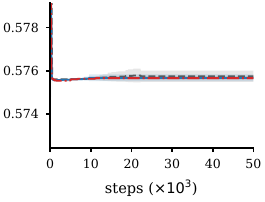}
    \caption{GISP, easy}
  \end{subfigure}\hfill
  \begin{subfigure}[t]{0.32\linewidth}
    \centering
    \includegraphics[width=\linewidth]{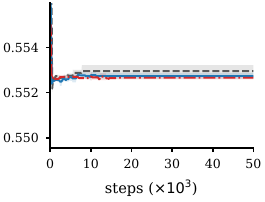}
    \caption{GISP, medium}
  \end{subfigure}\hfill
  \par\vspace{0.4em}
  \caption{\textbf{Fine-tuning on constraint activity prediction.} Part 1 of 2. Test KL of the checkpoint selected on validation,
  one panel per instance family, mean and one standard deviation over
  three seeds. Each panel carries its own vertical scale.}
  \label{fig:app-activity-1}
\end{figure}

\begin{figure}[p]
  \centering
  \begin{subfigure}[t]{0.32\linewidth}
    \centering
    \includegraphics[width=\linewidth]{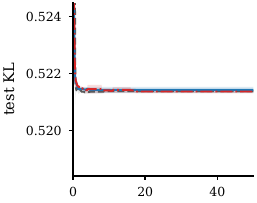}
    \caption{GISP, hard}
  \end{subfigure}\hfill
  \begin{subfigure}[t]{0.32\linewidth}
    \centering
    \includegraphics[width=\linewidth]{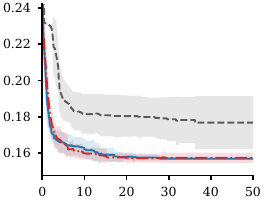}
    \caption{CFLP, easy}
  \end{subfigure}\hfill
  \begin{subfigure}[t]{0.32\linewidth}
    \centering
    \includegraphics[width=\linewidth]{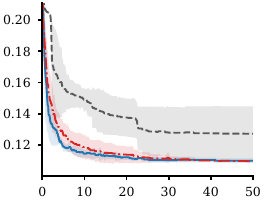}
    \caption{CFLP, medium}
  \end{subfigure}\hfill
  \par\vspace{0.4em}
  \begin{subfigure}[t]{0.32\linewidth}
    \centering
    \includegraphics[width=\linewidth]{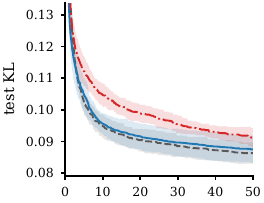}
    \caption{OTS, easy}
  \end{subfigure}\hfill
  \begin{subfigure}[t]{0.32\linewidth}
    \centering
    \includegraphics[width=\linewidth]{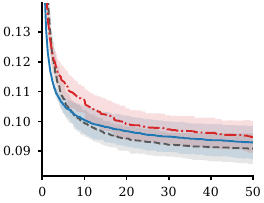}
    \caption{OTS, medium}
  \end{subfigure}\hfill
  \begin{subfigure}[t]{0.32\linewidth}
    \centering
    \includegraphics[width=\linewidth]{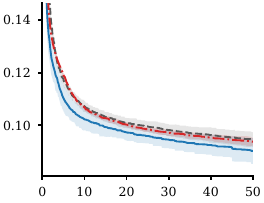}
    \caption{OTS, hard}
  \end{subfigure}\hfill
  \par\vspace{0.4em}
  \begin{subfigure}[t]{0.32\linewidth}
    \centering
    \includegraphics[width=\linewidth]{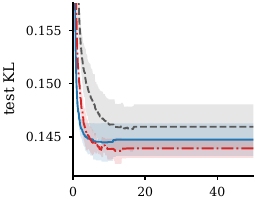}
    \caption{MMCN, medium-BI}
  \end{subfigure}\hfill
  \begin{subfigure}[t]{0.32\linewidth}
    \centering
    \includegraphics[width=\linewidth]{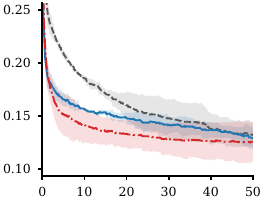}
    \caption{MMCN, medium-BC}
  \end{subfigure}\hfill
  \begin{subfigure}[t]{0.32\linewidth}
    \centering
    \includegraphics[width=\linewidth]{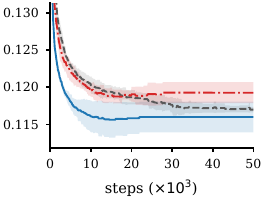}
    \caption{MMCN, hard-BI}
  \end{subfigure}\hfill
  \par\vspace{0.4em}
  \begin{subfigure}[t]{0.32\linewidth}
    \centering
    \includegraphics[width=\linewidth]{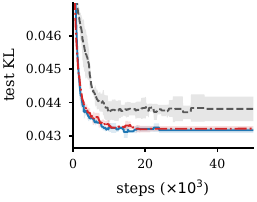}
    \caption{NNV, easy}
  \end{subfigure}\hfill
  \begin{subfigure}[t]{0.32\linewidth}
    \centering
    \includegraphics[width=\linewidth]{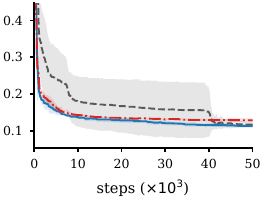}
    \caption{LB, hard}
  \end{subfigure}\hfill
  \caption{\textbf{Fine-tuning on constraint activity prediction.} Part 2 of 2. Test KL of the checkpoint selected on validation,
  one panel per instance family, mean and one standard deviation over
  three seeds. Each panel carries its own vertical scale.}
  \label{fig:app-activity-2}
\end{figure}

\end{document}

%% file: math_commands.tex
\usepackage{amsmath,amsfonts,bm}

\def\eqref#1{equation~\ref{#1}}

\def\1{\bm{1}}

\def\vzero{{\bm{0}}}

\def\vtheta{{\bm{\theta}}}
\def\va{{\bm{a}}}
\def\vb{{\bm{b}}}
\def\vc{{\bm{c}}}
\def\vd{{\bm{d}}}
\def\ve{{\bm{e}}}

\def\vg{{\bm{g}}}
\def\vh{{\bm{h}}}

\def\vq{{\bm{q}}}
\def\vr{{\bm{r}}}
\def\vs{{\bm{s}}}
\def\vt{{\bm{t}}}
\def\vu{{\bm{u}}}
\def\vv{{\bm{v}}}
\def\vw{{\bm{w}}}
\def\vx{{\bm{x}}}
\def\vy{{\bm{y}}}
\def\vz{{\bm{z}}}

\def\evc{{c}}

\def\mA{{\bm{A}}}

\def\mD{{\bm{D}}}

\def\mG{{\bm{G}}}
\def\mH{{\bm{H}}}
\def\mI{{\bm{I}}}

\def\mK{{\bm{K}}}
\def\mL{{\bm{L}}}

\def\mN{{\bm{N}}}

\def\mS{{\bm{S}}}
\def\mT{{\bm{T}}}
\def\mU{{\bm{U}}}

\def\mW{{\bm{W}}}
\def\mX{{\bm{X}}}
\def\mY{{\bm{Y}}}
\def\mZ{{\bm{Z}}}

\def\mLambda{{\bm{\Lambda}}}

\DeclareMathAlphabet{\mathsfit}{\encodingdefault}{\sfdefault}{m}{sl}
\SetMathAlphabet{\mathsfit}{bold}{\encodingdefault}{\sfdefault}{bx}{n}

\def\gC{{\mathcal{C}}}

\def\gI{{\mathcal{I}}}
\def\gJ{{\mathcal{J}}}

\def\gM{{\mathcal{M}}}
\def\gN{{\mathcal{N}}}

\def\gS{{\mathcal{S}}}
\def\gT{{\mathcal{T}}}

\def\gZ{{\mathcal{Z}}}

\def\sR{{\mathbb{R}}}

\def\sV{{\mathbb{V}}}

\def\sZ{{\mathbb{Z}}}

\def\emA{{A}}



%% file: fig_remilp_views.tex
\definecolor{viewink}{HTML}{34434B}
\definecolor{viewblue}{HTML}{326D91}
\definecolor{viewbluepale}{HTML}{EAF3F8}
\definecolor{vieworange}{HTML}{C27328}
\definecolor{vieworangepale}{HTML}{FFF1E1}
\definecolor{viewpurple}{HTML}{7652A0}
\definecolor{viewpurplepale}{HTML}{F2ECF9}
\definecolor{viewcoral}{HTML}{C96561}
\definecolor{viewcoralpale}{HTML}{FBEDEC}
\definecolor{viewteal}{HTML}{357E78}
\definecolor{viewtealpale}{HTML}{E9F5F2}
\begin{tikzpicture}[
  x=1cm,y=1cm,>=Latex,font=\sffamily\small,
  flow/.style={-{Latex[length=2.5mm]},line width=.95pt,draw=viewink!75},
  stepbadge/.style={circle,draw=viewink,fill=viewink,text=white,
    minimum size=5.6mm,inner sep=0pt,font=\sffamily\scriptsize\bfseries},
  cnode/.style={circle,draw=viewcoral,fill=viewcoralpale,minimum size=4.7mm,inner sep=0pt,font=\sffamily\scriptsize},
  newcnode/.style={cnode,dashed,fill=white},
  vnode/.style={circle,draw=viewblue,fill=viewbluepale,
    minimum size=4.7mm,inner sep=0pt,font=\sffamily\scriptsize}
]
\node[draw=vieworange!75,fill=vieworangepale,rounded corners=7pt,
  minimum width=2.80cm,minimum height=2.25cm] (base) at (1.62,3.52) {};
\node[font=\sffamily\bfseries,text=viewink] at (1.62,4.25) {Input MILP $M_0$};
\node[align=left] at (1.62,3.41)
  {$\begin{aligned}
    \min_{\vx\in\sV}\quad &\vc^{\T}\vx\\[-1pt]
    \text{s.t.}\quad &\mA\vx\le\vb
  \end{aligned}$};

\node[draw=viewblue!65,fill=viewbluepale,rounded corners=7pt,
  minimum width=5.80cm,minimum height=2.40cm] (ref) at (7.06,5.22) {};
\node[font=\sffamily\bfseries,text=viewblue] at (7.06,6.04)
  {Reference view $M$};
\node[align=left] at (7.06,5.07)
  {$\begin{aligned}
    \min_{\vx\in\sV}\quad &\vc^{\T}\vx\\[-1pt]
    \text{s.t.}\quad &(\mLambda\mA)\vx\le\mLambda\vb+\vs
  \end{aligned}$};

\node[draw=viewpurple!65,fill=viewpurplepale,rounded corners=7pt,
  minimum width=5.80cm,minimum height=2.40cm] (trans) at (7.06,1.82) {};
\node[font=\sffamily\bfseries,text=viewpurple] at (7.06,2.64)
  {Transformed view $M'$};
\node[align=left] at (7.06,1.68)
  {$\begin{aligned}
    \min_{\vy\in\sV}\quad &(\textcolor{vieworange}{\mT'}^{\T}\vc)^{\T}\vy\\[-1pt]
    \text{s.t.}\quad &(\mLambda'\mA\textcolor{vieworange}{\mT'})\vy\\[-1pt]
      &\le\mLambda'(\vb-\mA\textcolor{vieworange}{\vq'})+\vs'
  \end{aligned}$};

\draw[draw=viewink!75,line width=.95pt]
  (base.east)--(3.59,3.52) (3.59,1.82)--(3.59,5.22);
\draw[flow] (3.59,5.22)--(ref.west);
\draw[flow] (3.59,1.82)--(trans.west);

\node[draw=viewblue!45,fill=white,rounded corners=7pt,
  minimum width=4.19cm,minimum height=2.70cm] (gref) at (12.74,5.22) {};
\node[draw=viewpurple!45,fill=white,rounded corners=7pt,
  minimum width=4.19cm,minimum height=2.70cm] (gtrans) at (12.74,1.82) {};
\draw[flow] (ref.east)--(gref.west);
\draw[flow] (trans.east)--(gtrans.west);

\node[font=\sffamily\bfseries,text=viewblue] at (12.74,6.09)
  {Bipartite graph $\mathcal G(M)$};
\node[font=\sffamily\scriptsize,text=viewink!75] at (11.67,5.74)
  {constraints};
\node[font=\sffamily\scriptsize,text=viewink!75] at (13.82,5.74)
  {variables};
\coordinate (rc1) at (11.67,5.35);
\coordinate (rc2) at (11.67,4.81);
\coordinate (rr) at (11.67,4.27);
\coordinate (ri) at (13.82,5.35);
\coordinate (rj) at (13.82,4.81);
\coordinate (rk) at (13.82,4.27);
\draw[viewink!52,line width=.68pt]
  (rc1)--(ri) (rc1)--(rj) (rc2)--(ri) (rc2)--(rk)
  (rr)--(ri) (rr)--(rj) (rr)--(rk);
\node[cnode] at (rc1) {$c_1$};
\node[cnode] at (rc2) {$c_2$};
\node[newcnode] at (rr) {$r$};
\node[vnode] at (ri) {$i$};
\node[vnode] at (rj) {$j$};
\node[vnode] at (rk) {$k$};

\node[font=\sffamily\bfseries,text=viewpurple] at (12.74,2.69)
  {Bipartite graph $\mathcal G(M')$};
\node[font=\sffamily\scriptsize,text=viewink!75] at (11.67,2.34)
  {constraints};
\node[font=\sffamily\scriptsize,text=viewink!75] at (13.82,2.34)
  {variables};
\coordinate (tc1) at (11.67,1.95);
\coordinate (tc2) at (11.67,1.41);
\coordinate (tr) at (11.67,.87);
\coordinate (ti) at (13.82,1.95);
\coordinate (tj) at (13.82,1.41);
\coordinate (tk) at (13.82,.87);
\draw[viewink!52,line width=.68pt]
  (tc1)--(ti) (tc1)--(tj) (tc2)--(ti) (tc2)--(tk)
  (tr)--(ti) (tr)--(tj);
\draw[vieworange,line width=1.65pt] (tc2)--(tj);
\node[cnode] at (tc1) {$c'_1$};
\node[cnode] at (tc2) {$c'_2$};
\node[newcnode] at (tr) {$r'$};
\node[vnode] at (ti) {$i$};
\node[vnode] at (tj) {$j$};
\node[vnode] at (tk) {$k$};

\node[draw=viewteal!75,fill=viewtealpale,rounded corners=7pt,
  minimum width=2.30cm,minimum height=2.15cm] (encoder) at (16.66,3.52) {};
\node[font=\sffamily\bfseries,text=viewteal] at (16.66,4.23)
  {GAT encoder};
\draw[-{Latex[length=1.7mm]},draw=viewteal,line width=.9pt]
  (16.91,3.77)--(16.41,3.77);
\draw[-{Latex[length=1.7mm]},draw=viewteal,line width=.9pt]
  (16.41,3.26)--(16.91,3.26);
\foreach \yy in {3.77,3.26} {
  \node[cnode,minimum size=4.5mm,font=\sffamily\scriptsize] at (16.19,\yy) {$C$};
  \node[vnode,minimum size=4.5mm,font=\sffamily\scriptsize] at (17.13,\yy) {$V$};
}
\node[font=\sffamily\scriptsize,text=viewteal] at (16.66,2.80)
  {$\times 2$ layers, shared $f_\vtheta$};
\draw[flow] (gref.east)--(15.15,5.22)--(15.15,3.94)--(15.51,3.94);
\draw[flow] (gtrans.east)--(15.15,1.82)--(15.15,3.10)--(15.51,3.10);

\node[font=\sffamily\bfseries,text=viewink] at (20.18,6.35)
  {Variable embeddings};
\node[draw=viewblue!65,fill=viewbluepale,rounded corners=7pt,
  minimum width=2.70cm,minimum height=1.55cm] (href) at (20.18,5.22) {};
\node[draw=viewpurple!65,fill=viewpurplepale,rounded corners=7pt,
  minimum width=2.70cm,minimum height=1.55cm] (htrans) at (20.18,1.82) {};
\draw[flow] (17.81,3.94)--(18.18,3.94)--(18.18,5.22)--(href.west);
\draw[flow] (17.81,3.10)--(18.18,3.10)--(18.18,1.82)--(htrans.west);
\node[font=\sffamily\bfseries,text=viewblue] at (20.18,5.68)
  {$\mH_\vtheta(M)$};
\node[font=\sffamily\bfseries,text=viewpurple] at (20.18,2.28)
  {$\mH_\vtheta(M')$};
\foreach \rowindex/\yy/\lab in {0/5.31/i,1/5.07/j,2/4.83/k} {
  \node[font=\sffamily\scriptsize,text=viewink] at (19.43,\yy) {$\lab$};
  \foreach \colindex in {0,...,4} {
    \pgfmathtruncatemacro{\tone}{22+mod((\rowindex+2)*(\colindex+3),5)*12}
    \draw[fill=viewblue!\tone,draw=viewblue!75,line width=.2pt]
      ({19.66+.29*\colindex},{\yy-.09}) rectangle ++(.25,.18);
  }
}
\foreach \rowindex/\yy/\lab in {0/1.91/i,1/1.67/j,2/1.43/k} {
  \node[font=\sffamily\scriptsize,text=viewink] at (19.43,\yy) {$\lab$};
  \foreach \colindex in {0,...,4} {
    \pgfmathtruncatemacro{\tone}{22+mod((\rowindex+3)*(\colindex+2)+1,5)*12}
    \draw[fill=viewpurple!\tone,draw=viewpurple!75,line width=.2pt]
      ({19.66+.29*\colindex},{\yy-.09}) rectangle ++(.25,.18);
  }
}
\draw[draw=viewblue,line width=.75pt,rounded corners=.5pt]
  (19.62,5.20) rectangle (21.11,5.42);
\draw[draw=viewblue,line width=.75pt,rounded corners=.5pt]
  (19.62,4.96) rectangle (21.11,5.18);
\draw[draw=viewpurple,line width=.75pt,rounded corners=.5pt]
  (19.62,1.80) rectangle (21.11,2.02);
\node[font=\sffamily\scriptsize,text=viewblue] at (20.18,4.59) {$n\times d$};
\node[font=\sffamily\scriptsize,text=viewpurple] at (20.18,1.19) {$n\times d$};

\node[draw=viewink!30,fill=viewink!3,rounded corners=9pt,
  minimum width=21.12cm,minimum height=5.37cm] at (10.87,-3.115) {};
\draw[draw=viewink!45,dashed,line width=.75pt]
  (href.east)--(21.77,5.22)
  (htrans.east)--(21.77,1.82);
\draw[-{Latex[length=1.7mm]},draw=viewink!45,dashed,line width=.75pt]
  (21.77,5.22)--(21.77,-.14)--(20.94,-.14)--(20.94,-.43);
\node[font=\sffamily\scriptsize,text=viewink!65] at (19.64,-.14)
  {selected rows};
\node[anchor=west,font=\sffamily\bfseries,text=viewink]
  at (1.02,-.70) {Local prediction for variable $i\in\gI$};

\node[draw=vieworange!70,fill=vieworangepale,rounded corners=5pt,
  minimum width=6.05cm,minimum height=.59cm] (params) at (3.51,-1.42)
  {$T'_{ii}=u_i,\quad T'_{i,j(i)}=\lambda_i,\quad q'_i=\mu_i$};
\node[draw=viewblue!65,fill=viewbluepale,rounded corners=5pt,
  minimum width=3.66cm,minimum height=.59cm] (partner) at (3.51,-2.30)
  {$\sg\!\left(\vh_{j(i)}(M)\right)$};
\node[font=\sffamily\scriptsize,text=viewink!65] at (3.51,-2.80)
  {row $j(i)$ of $\mH_\vtheta(M)$, stop-gradient};
\node[inner sep=0pt,minimum width=2.86cm,minimum height=1.26cm]
  (hyper) at (9.13,-1.86) {};
\path[draw=viewteal!75,fill=viewtealpale,line width=.8pt]
  (7.70,-2.49)--(9.93,-2.49)
  arc[start angle=-90,end angle=90,radius=.63cm]
  --(7.70,-1.23)--cycle;
\node[font=\sffamily\bfseries,text=viewteal] at (9.13,-1.61)
  {Hypernetwork};
\node[font=\sffamily\large,text=viewteal] at (9.13,-2.03)
  {$\varrho_\phi$};
\draw[flow] (params.east)--(7.70,-1.42);
\draw[flow] (partner.east)--(7.70,-2.30);

\node[draw=viewteal!70,fill=white,rounded corners=5pt,
  minimum width=2.70cm,minimum height=1.28cm] (operator) at (12.45,-1.86) {};
\node[font=\sffamily\scriptsize,text=viewink] at (12.45,-.97)
  {$t_i\coloneq\bigl(u_i,\lambda_i,\mu_i,\vh_{j(i)}(M)\bigr)$};
\node[font=\sffamily\small,text=viewteal] at (12.45,-1.46)
  {$\varrho_\phi(t_i)$};
\foreach \rowindex in {0,...,2} {
  \foreach \colindex in {0,...,2} {
    \pgfmathtruncatemacro{\tone}{24+mod((\rowindex+1)*(\colindex+2),3)*17}
    \draw[fill=viewteal!\tone,draw=viewteal!75,line width=.2pt]
      ({11.60+.20*\colindex},{-1.87-.20*\rowindex}) rectangle ++(.16,-.16);
  }
}
\draw[draw=viewteal!80,line width=.6pt]
  (11.55,-1.85)--(11.48,-1.85)--(11.48,-2.45)--(11.55,-2.45)
  (12.18,-1.85)--(12.25,-1.85)--(12.25,-2.45)--(12.18,-2.45);
\node[font=\sffamily\scriptsize,text=viewteal] at (13.07,-2.16)
  {$\in\sR^{d\times d}$};
\draw[flow] (hyper.east)--(operator.west);

\node[draw=viewblue!65,fill=viewbluepale,rounded corners=5pt,
  minimum width=2.75cm,minimum height=.70cm] (sourceemb) at (8.83,-3.52)
  {$\vh_i(M)$};
\node[font=\sffamily\scriptsize,text=viewblue] at (8.83,-4.06)
  {row $i$ of $\mH_\vtheta(M)$};
\node[draw=viewteal!70,fill=viewtealpale,circle,
  minimum size=.73cm] (applyop) at (12.45,-3.52) {$\cdot$};
\node[draw=viewteal!70,fill=white,rounded corners=5pt,
  minimum width=3.42cm,minimum height=.70cm] (prediction) at (15.13,-3.52)
  {$\vz_{i,0}=\varrho_\phi(t_i)\cdot\vh_i(M)$};
\node[draw=viewpurple!65,fill=viewpurplepale,rounded corners=5pt,
  minimum width=2.44cm,minimum height=.70cm] (observed) at (19.35,-3.52)
  {$\vz_i=\vh_i(M')$};
\node[font=\sffamily\scriptsize,text=viewpurple] at (19.35,-4.06)
  {row $i$ of $\mH_\vtheta(M')$};
\draw[flow] (sourceemb.east)--(applyop.west);
\draw[flow] (operator.south)--(applyop.north);
\draw[flow] (applyop.east)--(prediction.west);
\draw[<->,draw=viewpurple!80,line width=.9pt]
  (prediction.east)--(observed.west)
  node[midway,above,font=\sffamily\scriptsize,text=viewink] {align};

\node[font=\sffamily\bfseries,text=viewink] at (17.58,-.77)
  {Contrastive candidates};
\node[draw=vieworange!70,fill=vieworangepale,rounded corners=5pt,
  minimum width=6.50cm,minimum height=.80cm] (actionneg) at (17.58,-1.55) {};
\node[draw=viewpurple!65,fill=viewpurplepale,rounded corners=5pt,
  minimum width=6.50cm,minimum height=.72cm] (variableneg) at (17.58,-2.42) {};
\node[font=\sffamily\scriptsize\bfseries,text=vieworange] at (17.58,-1.32)
  {Other descriptors ($K$)};
\node[font=\sffamily\scriptsize,text=viewink] at (17.58,-1.68)
  {$\bigl\{\varrho_\phi(t_{\alpha_k(i)})\cdot\vh_i(M)\bigr\}_{k=1}^{K}$};
\node[font=\sffamily\scriptsize\bfseries,text=viewpurple] at (17.58,-2.26)
  {Other variables ($K$)};
\node[font=\sffamily\scriptsize,text=viewink] at (17.58,-2.54)
  {$\bigl\{\vh_{\beta_k(i)}(M')\bigr\}_{k=1}^{K}$};
\node[draw=vieworange,fill=vieworange,circle,minimum size=.31cm,
  inner sep=0pt,font=\sffamily\scriptsize,text=white] at (14.65,-1.40) {$-$};
\node[draw=viewpurple,fill=viewpurple,circle,minimum size=.31cm,
  inner sep=0pt,font=\sffamily\scriptsize,text=white] at (14.65,-2.27) {$-$};

\node[draw=viewink!65,fill=white,rounded corners=5pt,
  minimum width=3.82cm,minimum height=.60cm] (loss) at (17.58,-4.90)
  {Contrastive loss $\mathcal L(M,M')$};
\draw[draw=viewink!55,line width=.8pt]
  (actionneg.east)--(21.08,-1.55)
  (variableneg.east)--(21.08,-2.42)
  (observed.east)--(21.08,-3.52)
  (21.08,-1.55)--(21.08,-4.90);
\draw[flow] (21.08,-4.90)--(loss.east);
\draw[flow] ([xshift=7.5mm]prediction.south)--([xshift=-17mm]loss.north);

\node[stepbadge] at (1.62,5.25) {1};
\node[stepbadge] at (16.66,5.03) {2};
\node[stepbadge] at (.67,-.70) {3};
\node[stepbadge] at (14.73,-.77) {4};

\end{tikzpicture}